%% file: main.tex
\documentclass{article}

\usepackage{mdframed}
\usepackage{microtype}
\usepackage{graphicx}
\usepackage{array}
\usepackage{subfigure}
\usepackage{booktabs} % for professional tables
\usepackage{stfloats}
\usepackage{float}
\usepackage{enumitem}
\usepackage{multirow}
\usepackage{adjustbox}

\usepackage{hyperref}

\usepackage{ dsfont }

\newcommand{\E}{\mathbb{E}}
\newcommand{\mask}{0}

\newcounter{boxcounter}

\newcommand{\topk}{Top-\(K\) }

\newcommand{\prior}{p_{\textrm{prior}}}

\usepackage[accepted]{icml2025}

\usepackage{amsmath}
\usepackage{amssymb}
\usepackage{mathtools}
\usepackage{amsthm}

\usepackage[capitalize,noabbrev]{cleveref}

\theoremstyle{plain}
\newtheorem{theorem}{Theorem}[section]
\newtheorem{proposition}[theorem]{Proposition}

\newtheorem{example}[theorem]{Example}
\newtheorem{conjecture}[theorem]{Conjecture}
\theoremstyle{definition}
\newtheorem{definition}[theorem]{Definition}
\newtheorem{assumption}[theorem]{Assumption}

\theoremstyle{remark}

\newcommand{\msg}{\textrm{MS}}

\newcounter{theo}
\renewcommand{\thetheo}{\arabic{theo}}
\newenvironment{theo}[2][]{%
\refstepcounter{theo}%
\ifstrempty{#1}%
{\mdfsetup{%
frametitle={%
\tikz[baseline=(current bounding box.east),outer sep=0pt]
\node[anchor=east,rectangle,fill=blue!20]
{\strut Theorem~\thetheo};}}
}%
{\mdfsetup{%
frametitle={%
\tikz[baseline=(current bounding box.east),outer sep=0pt]
\node[anchor=east,rectangle,fill=blue!20]
{\strut ~#1};}}%
}%
\mdfsetup{innertopmargin=10pt,linecolor=blue!20,%
linewidth=2pt,topline=true,%
frametitleaboveskip=\dimexpr-\ht\strutbox\relax
}
\begin{mdframed}[]\relax%
\label{#2}}{\end{mdframed}}

\usepackage[textsize=tiny]{todonotes}

\newcommand{\loss}{\mathcal L}

\icmltitlerunning{Train for the Worst, Plan for the Best: Understanding Token Ordering in Masked Diffusions}

\begin{document}

\twocolumn[
\icmltitle{Train for the Worst, Plan for the Best: \\
Understanding Token Ordering in Masked Diffusions}

\icmlsetsymbol{equal}{*}

\begin{icmlauthorlist}
\icmlauthor{Jaeyeon Kim}{equal,yyy}
\icmlauthor{Kulin Shah}{equal,comp}
\icmlauthor{Vasilis Kontonis}{comp}
\icmlauthor{Sham Kakade}{yyy}
\icmlauthor{Sitan Chen}{yyy}
\end{icmlauthorlist}

\icmlaffiliation{yyy}{Harvard University}
\icmlaffiliation{comp}{University of Texas Austin}

\icmlcorrespondingauthor{Kulin Shah}{kulinshah@utexas.edu}
\icmlkeywords{Machine Learning, ICML}

\vskip 0.3in]

\printAffiliationsAndNotice{\icmlEqualContribution}

\begin{abstract}
In recent years, masked diffusion models (MDMs) have emerged as a promising alternative approach for generative modeling over discrete domains. Compared to autoregressive models (ARMs), MDMs trade off complexity at training time with flexibility at inference time. At training time, they must learn to solve an exponentially large number of infilling problems, but at inference time, they can decode tokens in essentially arbitrary order. In this work, we closely examine these two competing effects. On the training front, we theoretically and empirically demonstrate that MDMs indeed train on computationally intractable subproblems compared to their autoregressive counterparts. On the inference front, we show that a suitable strategy for adaptively choosing the token decoding order significantly enhances the capabilities of MDMs, allowing them to sidestep hard subproblems. On logic puzzles like Sudoku, we show that adaptive inference can boost solving accuracy in pretrained MDMs from $<7$\% to $\approx 90$\%, even outperforming ARMs with $7\times$ as many parameters and that were explicitly trained via teacher forcing to learn the right order of decoding. This shows that MDMs without knowledge of the correct token generation order during training and inference can outperform ARMs trained with knowledge of the correct token generation order. We also show the effectiveness of adaptive MDM inference on reasoning tasks such as coding and math on the 8B large language diffusion model (LLaDa 8B). 
\end{abstract}

\section{Introduction}

While diffusion models~\cite{ho2020denoising,song2021score} are now the dominant approach for generative modeling in continuous domains like image, video, and audio, efforts to extend this methodology to discrete domains like text and proteins~\cite{austin2023structured,lou2024discrete,hoogeboom2021argmax} remain nascent. Among numerous proposals, masked diffusion models (MDMs) \cite{lou2024discrete,sahoo2024simple,shi2025simplified} have emerged as a leading variant, distinguished by a simple and principled objective: to generate samples, learn to reverse a noise process which independently and randomly masks tokens.

In many applications, such as language modeling, masked diffusion models (MDMs) still underperform compared to autoregressive models (ARMs)~\cite{nie2024scaling,zheng2024maskeddiffusionmodelssecretly}, which instead learn to reverse a noise process that unmasks tokens sequentially from left to right. However, recent studies suggest that MDMs may offer advantages in areas where ARMs fall short, including reasoning \cite{nie2024scaling,kitouni2024factorization}, planning \cite{ye2024beyond}, and infilling \cite{gong2024scaling}. This raises a key question: what are the strengths and limitations of MDMs compared to ARMs, and on what type of tasks can MDMs be scaled to challenge the dominance of ARMs in discrete generative modeling?

To understand these questions, we turn a microscope to two key competing factors when weighing the merits of MDMs over ARMs:
\begin{itemize}[leftmargin=*,topsep=0pt,itemsep=0pt]
    \item \textbf{Complexity at training time}: MDMs face a more challenging training task by design. While ARMs predict the next token given an unmasked prefix, MDMs predict a token conditioned on a set of unmasked tokens in arbitrary positions. This inherently increases their training complexity.
    \item \textbf{Flexibility at inference time}: On the other hand, the sampling paths taken by an MDM are less rigid. Unlike the fixed left-to-right decoding of ARMs, MDMs decode tokens in random order at inference. Even more is possible: MDMs can be used to decode in \emph{any order} (including left-to-right). 
\end{itemize}
Therefore, we ask:
\begin{center}
    \emph{Are the benefits of inference flexibility for MDMs enough to outweigh the drawbacks of training complexity?}
\end{center}
In this work, we provide dual perspectives on this question.

\textbf{(1) Training for the worst.} \enspace First, we provide theoretical and empirical evidence that the overhead imposed by training complexity quantifiably impacts MDMs' performance. 

Theoretically, we show examples of simple data distributions with a natural left-to-right order, where ARMs can provably generate samples efficiently. In contrast, there are noise levels at which a large fraction of the corresponding subproblems solved by MDMs for these distributions are provably computationally intractable. Empirically, we validate this claim on real-world text data, known to have left-to-right order and show that the imbalance in training complexity across subproblems persists even in real-world text data (Fig.~\ref{fig:scaling_laws}, left). 

\paragraph{(2) Planning for the best.} While the above might appear to be bad news for MDMs, in the second part of this paper, we answer our guiding question in the affirmative by building upon the observation~\cite{chang2022maskgit, zheng2024reparameterized} that MDMs which can perfectly solve all masking subproblems can be used to decode in \emph{any} order.

In first part of the paper, we show that the imbalance in complexity across subproblems during the training of MDMs results in some of the subproblems being poorly trained and the vanilla MDM inference that unmasks tokens in random order results in evaluating the poorly trained marginals. Therefore, in place of vanilla MDM inference, we consider \emph{adaptive} strategies that carefully select which token to unmask next. 
Our key insight is that the adaptive strategies makes it possible to \emph{sidestep} the hard subproblems from training (Fig.~\ref{fig:main_fig}). In particular, we find that \textbf{even without modifying how MDMs are trained}, the resulting models' logits contain enough information to determine the right order in which to unmask. We show the effectiveness of the adaptive inference in solving logic puzzles, coding, math and infilling tasks. For example, on Sudoku puzzles, a simple adaptive strategy (Section~\ref{subsec:effective-design}) improves the accuracy of MDMs from $<7$\% to almost 90\%.

\paragraph{Advantage of MDMs over ARMs.} We show that the main effectiveness of MDMs lies in tasks that do not have the \emph{same} natural token generation order across all sequences (e.g., logic puzzles and reasoning tasks like coding and math). By carefully designing experiments on logic puzzles, we show that \textbf{MDMs without the knowledge of the correct token generation order during training and inference} can outperform \textbf{ARMs trained with the knowledge of the correct token generation order}. In particular, we show that MDMs that decide the correct token generation order during inference via adaptive strategies can outperform ARMs that are trained to learn the right token generation order via supervised teacher forcing~\cite{shah2024causal,lehnert2024beyond}.

\paragraph{Organization.} In Section~\ref{sec:2}, we provide preliminaries on MDMs and set notation. In Section~\ref{sec:hardness}, we 
examine MDM training and demonstrate the imbalance in computational intractability across subproblems. In Section~\ref{sec:inference}, we consider adaptive inference in MDMs and investigate its impact on likelihood modeling across various tasks.
 
\section{Masked Diffusion Models (MDM)} \label{sec:2}
In this section, we explain the framework of Masked Diffusion Models \cite{shi2025simplified,sahoo2024simple} and highlight its interpretation as an \emph{order-agnostic learner}. MDMs gradually add masking noise to the true discrete data and learn the marginal distribution of the induced reverse process. We formally define both the forward and reverse processes for MDMs below.

Let the distribution $p_{\rm{data}}$ on $\{1,\ldots,m\}^L$ be the data distribution over sequences of length $L$ and with vocabulary $\{1, \ldots, m\}$. We use $\mask$ to denote the ``mask'' token. 

\paragraph{Forward process.} For a given $x_0 \sim p_{\rm{data}}$ and a noise level $t \in [0,1]$, the forward process $x_t \sim q_{t|0}(\cdot \, | \, x_0)$ is a coordinate-independent masking process via $q_{t|0}(x_t | x_0) = \prod_{i=0}^{L-1} q_{t|0}(x_t^i | x_0^i)$,
where $$q_{t|0}(x_t^i \mid x_0^i) = \mathrm{Cat}\bigl(\alpha_t \mathbf{e}_{x_0^i} + (1-\alpha_t)\mathbf{e}_{\mask} \bigr).$$ Here, $\alpha_t$ is a predefined noise schedule satisfying $\alpha_0 \approx 1, \alpha_1 \approx 0$ and $\mathbf{e}_{x_0^i} \in \mathbb{R}^{m+1}$ is a one-hot vector corresponding to the value of token $x_0^i$. $\mathrm{Cat}(\pi)$ denotes the categorical distribution given by $\pi \in \Delta^{m}$. In other words, for each $i$-th coordinate, $x_t^i$ is masked to the mask token $\mask$ with probability $1-\alpha_t$ and remains unchanged otherwise.

\begin{figure}[t]
    \centering
    \includegraphics[width=0.47\textwidth]{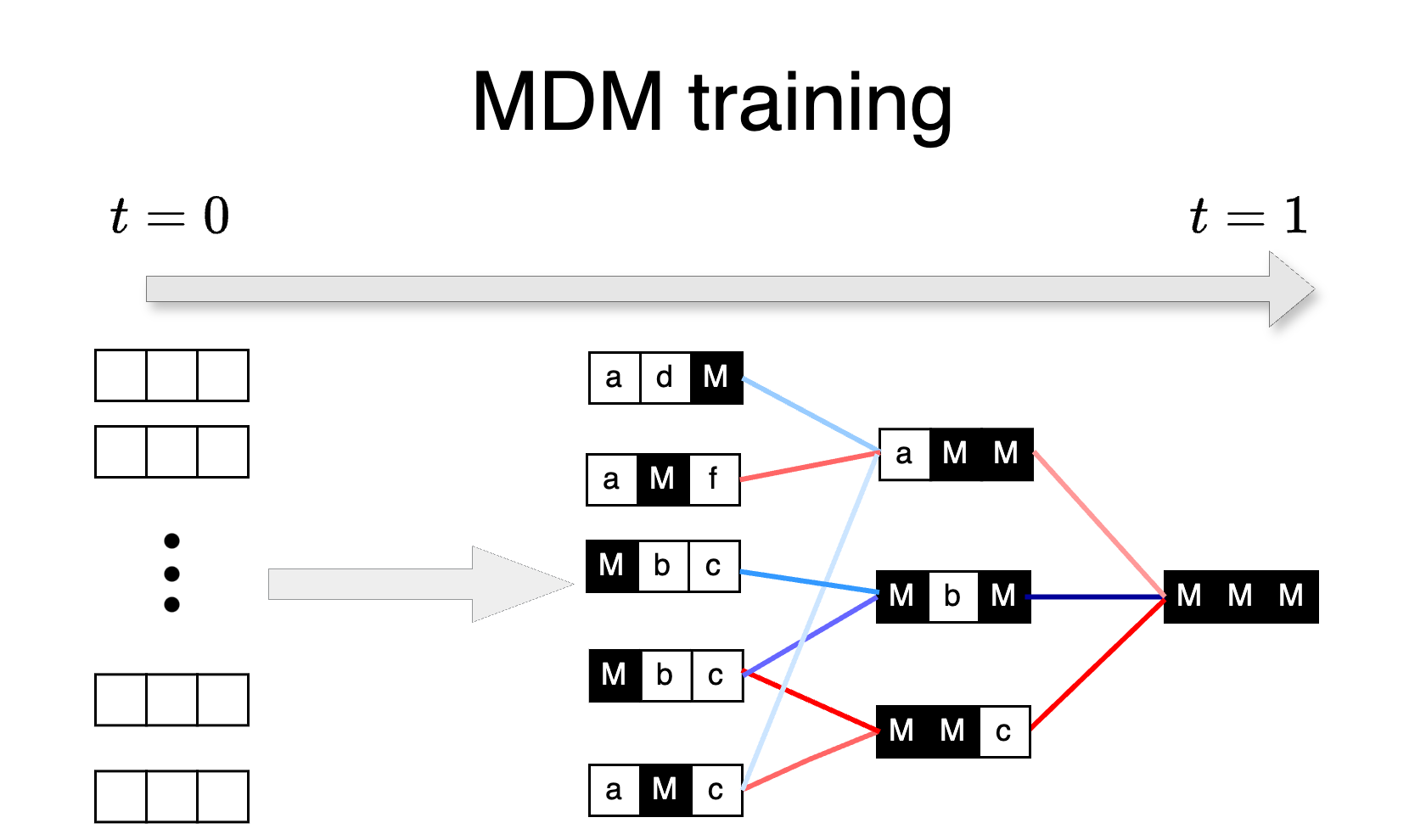}
    \vspace{-0.2in}
    \includegraphics[width=0.47\textwidth]{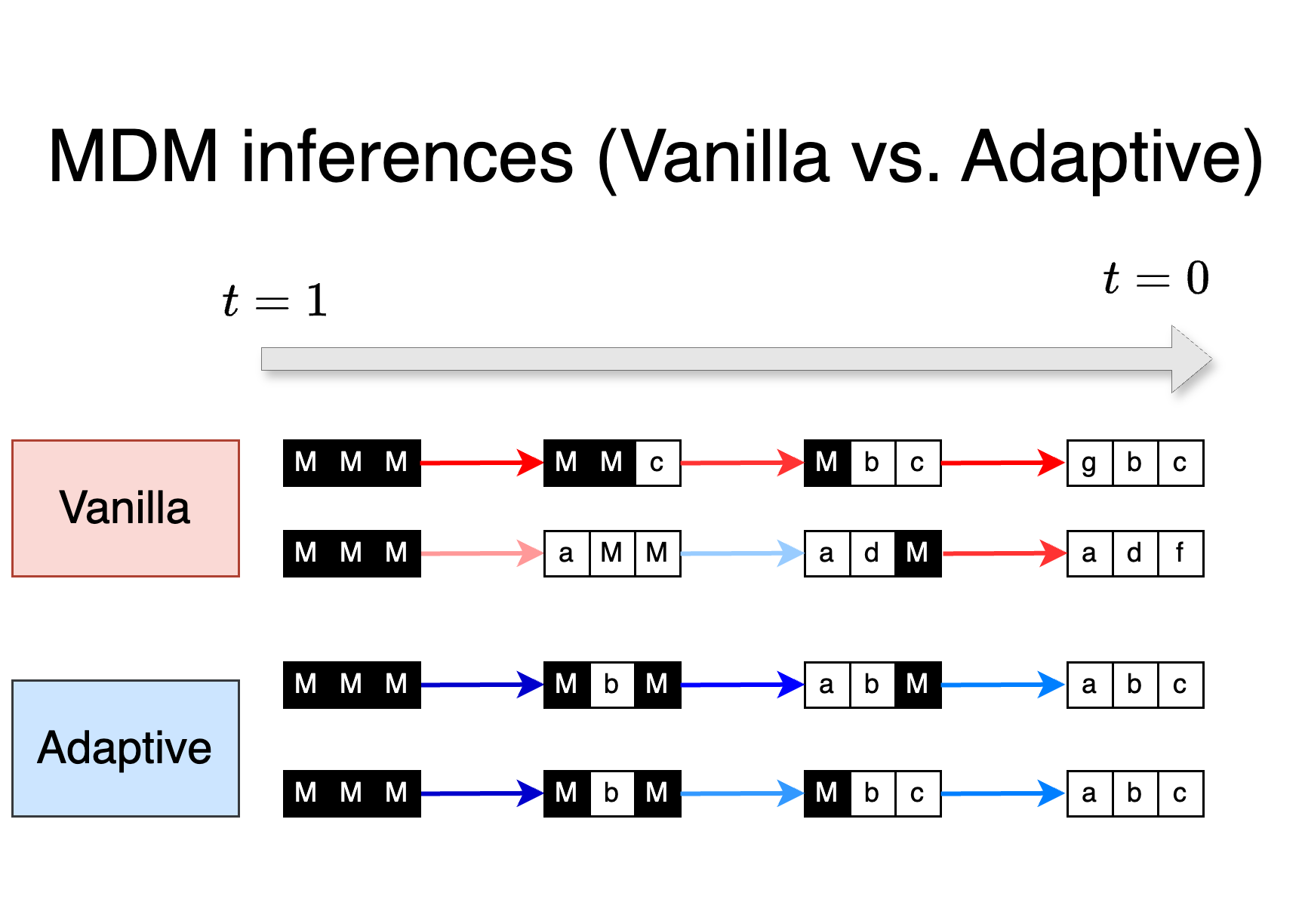}
    \caption{ 
    (\textbf{Top}) MDM training can be seen as learning multiple masked prediction problems, where some are harder to learn, leading to performance imbalance (Section~\ref{sec:hardness}). 
    (\textbf{Bottom}) During inference, adaptive MDM can avoid difficult problem instances, improving performance (Section~\ref{sec:inference}).}
    \label{fig:main_fig}
\end{figure}

\paragraph{Reverse process.} The reverse process of the above forward process is denoted by $q_{s|t}(x_s | x_t, x_0)$ and is given by $q_{s|t}(x_s | x_t, x_0) = \prod_{i=0}^{L-1} q_{s|t}(x_s^{i} | x_t, x_0)$ for any $s<t$, where
\begin{equation*}
     q_{s|t}(x_s^i \, \lvert\, x_t, x_0) = \begin{cases} 
       \mathrm{Cat}(\mathbf{e}_{x_t^{i}}) \quad & x_t^i \ne  0 \\
     \mathrm{Cat}\left(\frac{1-\alpha_s}{1-\alpha_t}\mathbf{e}_0 + \frac{\alpha_s - \alpha_t}{1-\alpha_t}\mathbf{e}_{x_0^i}\right)
  \quad &x_t^i= 0\,. 
     \end{cases}
\end{equation*}
%\kulin{changed $e_{x_0}$ to $e_{x_0^i}$.} 
The reverse transition probability $q_{s|t}(x_s^i | x_t, x_0)$ is approximated using $g_{\theta}(x_s^i | x_t) \triangleq q_{s|t}(x_s^i \, \lvert\, x_t, x_0 \leftarrow p_{\theta}(  \cdot | x_t, t) )$ where $p_\theta( \cdot | x_t,t)$ is a denoising network trained to predict the marginal distribution on $x_0^i$ via an ELBO-based loss for all masked tokens at noise scale $t$ (i.e., for all $i$ such that $x_t^i = 0$). To be precise, $q_{s|t} \left( x_s^i \mid x_t, x_0 \leftarrow p_{\theta}( \cdot | x_t, t) \right)$ indicates the conditional probability where $p_{\theta}(\cdot | x_t, t)$ is placed in the position of $e_{x_0^i}$ within $q_{s|t}(x_s^i \mid x_t, x_0)$. The denoising network is trained to minimize the following loss derived from the score-entropy \cite{lou2024discrete, sahoo2024simple, shi2025simplified, ou2024absorbing}:
\begin{equation*}
    \mathcal{L}_\theta = \int_{0}^1 \frac{\alpha_t'}{1-\alpha_t} \displaystyle \mathop{\mathbb{E}}_{ \substack{x_0 \sim p_{\rm data} \\ x_t \sim q_{t|0}(\cdot | x_0)}  }  
    \sum_{i: x_t^i = 0} -\log p_\theta(x_0^i | x_t,t)  dt,
\end{equation*}
where $\alpha_t'=\frac{d \alpha_t}{dt}$ and the summation is computed over masked tokens (i.e., all $i$ such that $x_t^i = \mask$). In practice, a time-embedding-free architecture for the denoising network, i.e., $p_\theta( \cdot | x_t, t) = p_\theta(\cdot | x_t)$ is generally used as \(x_t\) implicitly contains information about \(t\) via the number of masked tokens.

The reverse sampling process starts from the fully masked sentence $x_1 = (\mask,\ldots,\mask)$. Suppose we have a partially \textbackslash fully masked sequence \(x_t\) at a given noise level \(t \in (0,1]\). Then, to obtain $x_s$ for a predetermined noise level \(s < t\), we sample $x_s^i \sim g_\theta(\cdot | x_t)$ for all $i$. This process is repeated recursively from \(t=1\) to \(t=0\).

\subsection{Reformulating the training and inference of MDMs} \label{sec:agnostic_learner}

In this section, we first discuss training of MDMs and compare it with ``left-to-right" order training of autoregressive models in \cref{sec:vanilla-mdm-training}. Then, we reformulate vanilla MDM inference in \cref{sec:vanilla-mdm-inference} to set the stage for the upcoming discussion.

\subsubsection{Order-agnostic training of MDMs}
\label{sec:vanilla-mdm-training}

Recent works \cite{zheng2024maskeddiffusionmodelssecretly,ou2024absorbing} have observed that the learning problem of MDM is equivalent to a masked language model. Building upon their analysis, we reformulate the loss $\mathcal{L}_\theta$ to show that $\loss_{\theta}$ is a linear combination of the loss for all possible infilling masks. We first define \(x_0[M]\) as a masked sequence, obtained from original sequence $x_0$ where indices in the mask set $M$ (a subset of $[L]\triangleq\{1,2,\ldots,L\}$) are replaced with mask token $0$.

\begin{proposition} \label{prop:mdm_loss}
Assume $\alpha_0=1$, $\alpha_1 =0$ and denoising network $p_\theta$ is time-embedding free.
Then $ \mathcal{L}_\theta \le -\mathbb{E}_{x_0 \sim p_{\rm data}}[\log p_\theta(x_0)]$ and
\begin{equation} \label{eqn:mdm_loss}
\mathcal{L}_\theta = -\sum_{ M\subseteq [L],i \in M} \frac{1}{| M |} \frac{1}{\binom{L}{|M|}} \displaystyle \mathop{\mathbb{E}}_{x_0 \sim p_{\rm data}} [ \log p_\theta(x^i_0 | x_0[M]) ],
\end{equation}
where $|M|$ is the size of the set $M$ 
and \(p_\theta(x_i \mid x_0[M])\) indicates the conditional probability of the \(i\)-th coordinate from \(p_\theta(x_t)\).
\end{proposition}
The proof of the above proposition is given in Appendix~\ref{appenix:mdm-equivalent-loss}. As the MDM loss is a linear combination of the loss for all possible infilling mask $M$, the minimizer of the loss $\loss_{\theta}$ learns to solve \emph{every} masking problem. In other words, the optimal predictor $p_\theta$ is the posterior marginal of the $i$-th token, conditioned on $x_0[M]$ for all masks $M$.

On the other hand, Autoregressive Models (ARMs) learn to predict $i^{\textrm{th}}$ token $x^i$ based on all preceding tokens, from $x^0$ to $x^{i-1}$. This is equivalent to predicting $x^i$ by masking positions from $i$ to $L-1$. Therefore, the training objective for ARMs can be expressed as:
\begin{equation} \label{eqn:ar_loss}
    \log p_\theta(x_0) = \sum_{i=0}^{L-1} \log p_\theta ( x_0^i | x_0 [\{i,\ldots,L-1\}]).
\end{equation}
Typically, ARMs are trained to predict tokens sequentially from left to right. We refer to this as left-to-right training. However, it's also possible to train these models to predict tokens sequentially based on a \emph{fixed, known} permutation of the sequence. We refer to this general approach as \textbf{order-aware training}.

To understand the comparison between the training objective of MDMs and ARMs, we want to highlight the equivalence between any-order autoregressive loss and MDM loss \cite{hoogeboom2022autoregressive, ou2024absorbing}. In particular, under conditions of Proposition~\ref{prop:mdm_loss}, MDM loss is equal to 

{\small
\begin{align*}
    \mathcal{L}_\theta =- \mathop{\E}_{\substack{ x_0 \sim p_{\textrm{data}} \\ \pi \sim \textrm{Unif}(\mathbb S_L) }} \left[\sum_{i=0}^{L-1} \log p_\theta \left( x_0^{\pi(i)} \Big| x_0 [\pi\{i,\ldots,L-1\}] \right) \right],
\end{align*}}

where $\textrm{Unif}(\mathbb S_L)$ is a uniform distribution over all the permutations of length $L$ (See Appendix~\ref{sec:mdm-aoarm} for the proof). Observe that if the expectation is only with respect to the identity permutation, then the loss becomes an autoregressive loss. This shows that MDM loss solves \emph{exponentially} more subproblems than ARM loss. In contrast to ARM loss, MDM does not prefer any particular (e.g., left-to-right) order during the training; therefore, we call its training \emph{order-agnostic training}.

\subsubsection{Order-agnostic inference of MDMs}
\label{sec:vanilla-mdm-inference}

The MDM inference can be decomposed into two steps: (a) randomly selecting a set of positions to unmask and (b) assigning token values to each position via the denoising network $p_\theta$. More precisely, we can reformulate the reverse process $x_s \sim g_\theta(\cdot | x_t)$ as follows. 

\begin{theo}[Vanilla MDM inference]{alg:random_sampler_redefine}

\begin{itemize}
    \item[(a)] Sample a set of masked tokens \(\mathcal{S} \subseteq \{i \mid x_t^i = \mask\}\), \(\mathbb{P}(i \in \mathcal{S}) = \frac{\alpha_s-\alpha_t}{1-\alpha_t}\).
    \item[(b)] For each $i \in \mathcal{S}$, sample $x_s^i \sim p_\theta(x^i | x_t)$.
\end{itemize}
\end{theo}
Therefore, the inference in MDM is implemented by randomly selecting $S$ and then filling each token value according to the posterior probability $p_{\theta}(x_s^i | x_t)$. 

On the other hand, ARMs are trained to predict tokens sequentially from left to right and therefore, generate tokens also in left-to-right order. In contrast, vanilla MDM inference generates the tokens in a random order. 

\section{MDMs train on hard problems}\label{sec:hardness}

\begin{figure*}[t]
    \centering
    \includegraphics[width=0.8\textwidth]{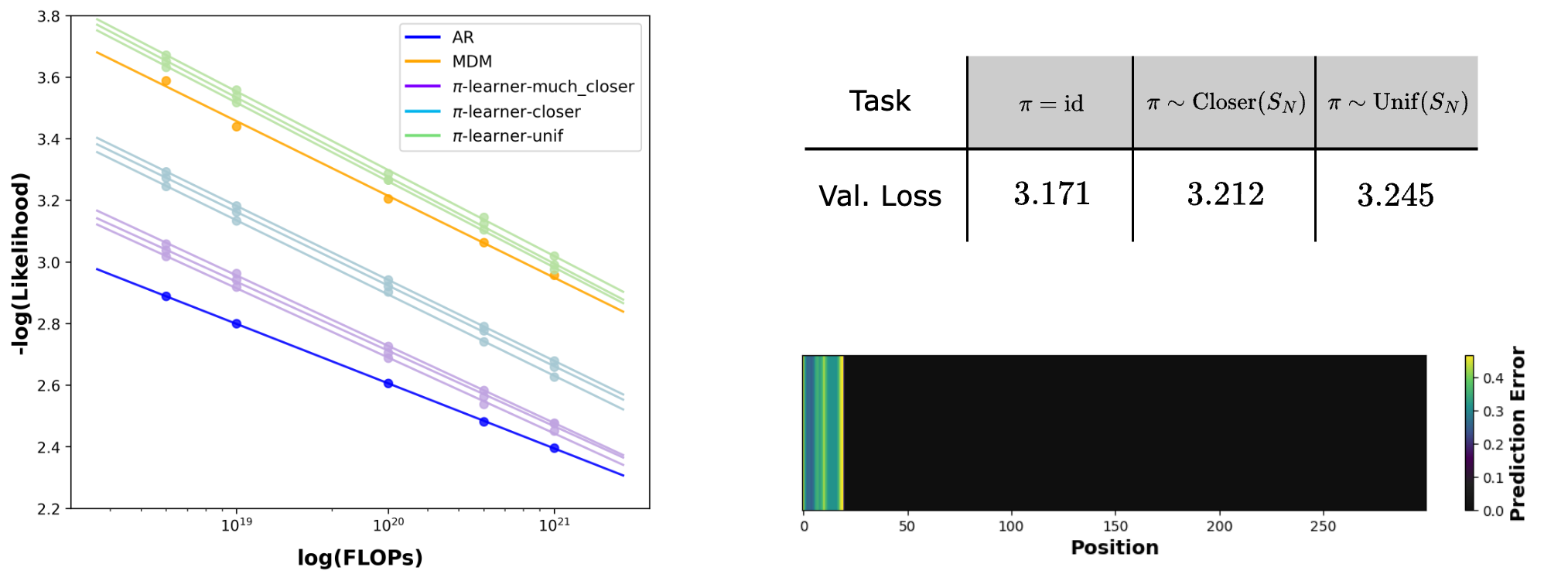} 
\caption{\textbf{Left: MDMs train on hard problems (Section~\ref{sec:hardness_text})}. x-axis and y-axis correspond to $\log(\text{FLOPs})$ and $-\log p_\theta(x)$, respectively. MDM {\color{blue} (Blue)} is worse than ARM {\color{orange} (Orange)} in likelihood modeling. Most masking problems {\color{purple} (Other lines)} that MDM is trained on are harder than those encountered by ARM, as indicated by small log-likelihoods. \textbf{Right: Task error imbalance (Section~\ref{sec:imbalance_error})}. MDM's performance varies across different tasks. For text data (top right), this is indicated by validation loss. For L\&O-NAE-SAT (bottom right), MDM performs well on the masking problems for observation positions (light region) but struggles with latent positions (dark region).}\label{fig:scaling_laws}
\end{figure*}

In this section, we provide theoretical and empirical evidence that when the data distribution has left-to-right order (or any fixed known order) then autoregressive training in left-to-right order (or in the known order) is more tractable than MDMs. In particular, for such distributions with fixed order, we show that ARMs can efficiently sample from the distributions but for MDMs, we theoretically and empirically demonstrate that a large portion of masking subproblems $p_\theta(x^i_0 \mid x_0[M])$ can be difficult to learn. 

In Section~\ref{sec:csp}, we show several examples of simple, non-pathological distributions for which: (1) the masking problems encountered during order-\emph{aware} training (such as in ARMs) are computationally tractable, yet (2) many of the ones encountered during order-agnostic training (such as in MDMs) are computationally intractable. In Section~\ref{sec:hardness_text}, we empirically show that text data also exhibits this gap between the computational complexity of order-aware and order-agnostic training and therefore, MDMs train on subproblems of wide variety of complexity (depending on the order/masks). In Section~\ref{sec:imbalance_error}, we empirically show that the variety in training complexity results in \underline{\textbf{performance imbalance across subproblems}}: MDMs trained on data from such distributions exhibits small errors on easy subproblems but suffers from large errors on harder ones.

\subsection{Benign distributions with hard masking problems} \label{sec:csp}

We now describe a simple model of data under which we explore the computational complexity of masking problems and show the contrast between masking problems encountered by MDMs and ARMs. 

\begin{definition}\label{definition:planted}
    A \emph{latents-and-observations (L\&O) distribution} is a data distribution $p_{\rm data}$ over sequence of length $L$ with alphabet size $m$ (precisely, $p_{\rm data}$ is over $\{0,\ldots,m\}^L$)  is specified by a permutation $\pi$ over indices $\{1, 2, \ldots, L \}$, number of latent tokens $N$, number of observation tokens $P$ such that $N + P = L$, prior distribution $\prior$ of latent variables over $\{1,\ldots,m\}$ and efficiently learnable \emph{observation functions} $\mathcal{O}_1,\ldots,\mathcal{O}_P: \{1,\ldots,m\}^N \to \Delta(\{0,\ldots,m\})$,\footnote{Here \emph{efficiently learnable} is in the standard PAC sense: given polynomially many examples of the form $(z,y)$ where $z\sim \prior^N$ and $y\sim \mathcal{O}_j(z)$, there is an efficient algorithm that can w.h.p. learn to approximate $\mathcal{O}_j$ in expectation over $\prior^N$.} 
    \begin{itemize}[topsep=0pt,itemsep=0pt,leftmargin=*]
        \item (\textbf{Latent tokens}) For $i = 1,\ldots,N$, sample $x^{\pi(i)}$ 
        independently from the prior distribution $\prior$ of the latents.
        \item (\textbf{Observation tokens}) For $j = 1,\ldots,P$, sample $x^{\pi(N + j)}$ independently from $\mathcal{O}_j(x^{\pi(1)},\ldots,x^{\pi(N)})$.
    \end{itemize} 
\end{definition}

L\&O distributions contain two types of tokens: (1) \emph{latent tokens} and (2) \emph{observation tokens}. Intuitively, latent tokens are tokens in the sequence, indexed by $\pi(1), \pi(2), \ldots, \pi(N)$ that serve as ``seeds" that provide randomness in the sequence; the remaining tokens, called observation tokens (indexed by $\pi(N+1), \pi(N+2), \ldots, \pi(N+P)$), are determined as (possibly randomized) functions of the latent tokens via $\mathcal{O}_1,\ldots,\mathcal{O}_P$.  Observe that L\&O distributions specified by a permutation $\pi$ have a natural generation order by permutation $\pi$.

\paragraph{Order-aware training} Order-aware training, i.e. by permuting the sequence so that $\pi$ becomes the identity permutation and then performing autoregressive training, is computationally tractable: predicting $x^{\pi(i)}$ given $x^{\pi(1)},\ldots,x^{\pi(i-1)}$ is trivial when $i \le N$ as the tokens are independent, and computationally tractable when $i > N$ because $x^{\pi(i)}$ only depends on $x^{\pi(1)},\ldots,x^{\pi(N)}$ and is efficiently learnable by assumption. In contrast, below we will show examples where if one performs order-agnostic training \emph{à la} MDMs, one will run into hard masking problems with high probability.

\paragraph{Order-agnostic training} We first note that if the observations $(\mathcal{O}_1,\ldots,\mathcal{O}_P)$ are given by a cryptographic hash function, then the masking problem of predicting $(x^{\pi(1)},\ldots,x^{\pi(L)})$ given $(x^{\pi(N+1)},\ldots,x^{\pi(N+P)})$ is computationally intractable by design because it requires inverting the hash function. While this is a well-known folklore observation regarding the role of token ordering in language modeling, it is not entirely satisfying because this construction is worst-case in nature \--- in real-world data, one rarely trains on sequences given by cryptographic hash functions. Furthermore, it only establishes hardness for a specific masking pattern which need not be encountered in the course of running the reverse process.

We provide several simple instances of L\&O distributions that address these issues: instead of leveraging delicate cryptographic constructions, they are \emph{average-case} in nature and furthermore we can establish hardness for \emph{typical} masking problems encountered along the reverse process. 

In all these examples, the hardness results we establish hold even if the algorithm knows all of the parameters of $p_{\rm data}$ as well as the observation functions $\mathcal{O}_1,\ldots,\mathcal{O}_P$. Due to space constraints, here we focus on the following example, deferring two others to Apps.~\ref{app:parity} and~\ref{app:slab}.

\begin{example}[Sparse predicate observations]\label{example:csp}
    Consider the following class of L\&O distributions. Given \emph{arity} $k\ge 2$, fix a \emph{predicate} function $g: \{1,\ldots,m\}^k \to \{0,1\}$. Consider 
    the set of all ordered subsets of $\{1,2,\ldots,N\}$ of size $k$ and set the total number of observation latents $P$ equal to the size of this set (hence $P = N ! / (N-k)! = N(N-1)\cdots(N-k+1)$). To sample a new sequence, we first sample latent tokens $x^{\pi(1)},\ldots,x^{\pi(N)}$ from the prior distribution $\prior$ and an observation latent corresponding to a $k$-sized subset $S$ is given by $g( \{ x^{\pi(i)} \}_{i \in S} )$. In other words, each observation latent corresponds to a $k$-sized subset $S$ of $\{1,2,\ldots,N\}$ and the corresponding observation function $\mathcal{O}_S(x^{\pi(1)}, \ldots, x^{\pi(N)} )$ is given by $g( \{ x^{\pi(i)} \}_{i \in S} )$.
\end{example}

\begin{proposition}\label{prop:csp}
    Let $x$ be a sample from an L\&O distribution $p_{\rm data}$ with sparse predicate observations as defined in Example~\ref{example:csp}, with arity $k$ and predicate $g$ satisfying Assumption~\ref{assume:paramagnetic}, and let $\gamma$ be the probability that $g$ is satisfied by a random assignment from $\{1,\ldots,m\}^k$. Let $D_{\rm KS}$ and $D_{\rm cond}$ be some constants associated with the predicate function $g$ (see Definition~\ref{def:thresholds}). Suppose each token in $x$ is independently masked with probability $\alpha$, and $M$ is the set of indices for the masked tokens. If $1 - \gamma^{-1} D_{\rm KS}/kN^{k-1} \le \alpha \le 1 - \gamma^{-1} D_{\rm cond}/kN^{k-1}$, then under the \emph{1RSB cavity prediction} (see Conjecture~\ref{conj:1rsb}), with probability $\Omega_k(1)$ over the randomness of the masking, no polynomial-time algorithm can solve the resulting subproblem of predicting any of the masked tokens among $x^{\pi(1)},\ldots,x^{\pi(N)}$ given $x[M]$.
\end{proposition}
The complete proof of the proposition is given in \Cref{app:planted_result}. We also provide a proof outline in \Cref{appendix:pf_outline_hardness} for a comprehensive understanding.

\subsection{Empirical evidence of hardness via likelihoods}
\label{sec:hardness_text}

In the previous section, we provided theoretical evidence that order-aware training is tractable when data has a natural order but the order-agnostic training is not. In this section, we provide empirical evidence to support this claim, using natural text data. Additionally, recent studies \cite{nie2024scaling, zheng2024maskeddiffusionmodelssecretly} have shown that masked diffusion models (MDMs) underperform compared to autoregressive models (ARMs) on natural text data. In this section, we provide evidence that this performance gap is primarily due to the order-agnostic training of MDMs. Natural text inherently follows a left-to-right token order, and we show that as training deviates from this order, model performance progressively declines.

To understand the importance of the order during the training, we use the following setting: Given a permutation $\pi$ of indices $\{0,1, \ldots, L-1 \}$, define a \emph{$\pi$-learner} to be a likelihood model $\log p_{\theta}(x_0)$ given as follows: 
\begin{equation}
\label{eq:pi-learner-likelihood}
    \log p_{\theta}(x_0) = \sum_{i=0}^{L-1} \log p_\theta \bigl( x_0^{\pi(i)} \Big| x_0 [\pi\{i,\ldots,L-1\}] \bigr)\,
\end{equation}
In other words, the $\pi$-learner predicts the token at position $\pi(i)$ given the clean tokens $x_0^{\pi(0)},\ldots, x_0^{\pi(i-1)}$ and masked tokens $x_0^{\pi(i)},\ldots, x_0^{\pi(L-1)}$. If $\pi$ is the identity permutation, this reduces to the standard (left-to-right) autoregressive training. Note that the MDM loss encodes a $\pi$-learner for every permutation $\pi$ because
the MDM loss~\eqref{eqn:mdm_loss} is equivalent to the average loss of those $\pi$-learners over $\pi$ sampled from $\mathrm{Unif}(\mathbb{S}_L)$:

{\small
\begin{align*}
    \mathcal{L}_\theta =- \mathop{\E}_{\substack{ x_0 \sim p_{\textrm{data}} \\ \pi \sim \textrm{Unif}(\mathbb S_L) }} \left[\sum_{i=0}^{L-1} \log p_\theta \left( x_0^{\pi(i)} \Big| x_0 [\pi\{i,\ldots,L-1\}] \right) \right],
\end{align*}}

where $\mathbb{S}_L$ denotes the set of all permutations over $\{0, 1, \ldots, L-1\}$. The proof of the above equivalence is given in \cref{appenix:mdm-equivalent-loss}. Therefore, by measuring the `hardness' of each $\pi$-learner, we can probe differences in hardness between arbitrary masking problems and left-to-right masking problems.

\paragraph{Experimental setup.} We use the Slimpajama dataset \cite{soboleva2023slimpajama} to evaluate the performance of training in different orders. To train a $\pi$-learner, we employ a transformer with causal attention and use permuted data $\pi(x_0)$ as input. By varying $\pi$ while maintaining all other training configurations (e.g., model, optimization), we can use the resulting likelihood (computed using \cref{eq:pi-learner-likelihood}) as a metric to capture the hardness of subproblems solved by the $\pi$-learner.

In our experiments, the sequence length $L$ is $2048$, so repeating the scaling laws for each $\pi$ is infeasible. Instead, we sample $\pi \sim \mathrm{Unif}(\mathbb{S}_L)$ and examine the scaling law of the $\pi$-learner's likelihood. We leverage the codebase from \cite{nie2024scaling}, where the baseline scaling laws of MDM and ARM were introduced. Moreover, given that RoPE has an inductive bias towards left-to-right ordering, we employ a learnable positional embedding layer for all experiments to correct this. Consequently, we also re-run the baseline results, where RoPE was employed. To investigate how the distance between $\pi$ and the identity permutation affects the scaling law, we consider two interpolating distributions over permutations between $\mathrm{Unif}(\mathbb{S}_L)$ (i.e, MDM training) and the point mass at the identical permutation (i.e, ARM training). We sample three permutations from the interpolating distribution and $\mathrm{Unif}(\mathbb{S}_L)$ and plot the scaling law for each of the permutation. Due to space constraints, we provide further experimental details in Appendix~\ref{appendix:exp_detail_text}.

\paragraph{Results.} As shown in Fig.~\ref{fig:scaling_laws}, the scaling law for a $\pi$-learner with uniformly random $\pi$ is worse than that of an ARM. This elucidates the inherent hardness of masking problems \( p_\theta(x_i \mid x_0[M]) \) beyond left-to-right prediction and also explains why MDM, which is trained simultaneously on all $\pi \in \mathbb{S}_L$, is worse than ARM in likelihood modeling. Additionally, as $\pi$ gets closer to the identity permutation, the scaling laws also get closer to ARM ($\pi$-learner-closer and $\pi$-learner-much-closer in Fig.~\ref{fig:scaling_laws}). This also supports the common belief that ARM is a good fit for text data as it inherently follows a \emph{left-to-right} ordering.

That said, it should also be noted that even though MDMs are trained on exponentially more masking problems than ARM ($\Theta(L2^L)$ versus $L$), its performance is not significantly worse than $\pi$-learners. We attribute this to the \emph{blessing of task diversity};  multi-task training can benefit both the optimization dynamics \cite{kim2024task} and validation performance \cite{tripuraneni2022provablem,andreas2016benefit,ruder2017overview} due to positive transfers across tasks.

\subsection{Error is imbalanced across masking problems}
\label{sec:imbalance_error}
In previous sections, we have demonstrated that the hardness of different masking problems \( p_\theta(x^i \mid x_0[M]) \) can vary significantly, potentially hindering the MDM's learning. In this section, we provide empirical evidence that the MDM's final performance exhibits a similar imbalance across subproblems. Details are provided in App.~\ref{appendix:exp_detail_3_3}.

\paragraph{L\&O-NAE-SAT.}
Consider an L\&O distribution with $\pi$ given by the identity permutation and where each observation $\mathcal{O}_j$ is deterministically given by $\mathrm{NAE}(x_{i_1},x_{i_2},x_{i_3}) \triangleq 1 - \mathbf{1}[x_{i_1} = x_{i_2} = x_{i_3}]$ for some randomly chosen (prefixed) triples $(i_1,i_2,i_3) \in[N]$. For an MDM trained on this distribution, we measure the error it achieves on each task $\log p_\theta(x_0 | x_0[M])$ via $ \mathbb{E}_{x_0} \Bigl \| \log p_\theta(x_0 | x_0[M])-  \log p_{\rm data}(x_0 | x_0[M]) \Bigr\|^2$,
where $p_{\rm data}(x_0 | x_0[M])$ denotes the Bayes-optimal predictor.
Technically, we do not have access to this, so instead we train another MDM for a much larger number of iterations and use this as a proxy. Fig.~\ref{fig:scaling_laws} reveals that prediction tasks for latent positions (light region) exhibit larger errors compared to those for observation positions (dark region). 
 
\paragraph{Text.}
Here we revisit the text experiment from Section~\ref{sec:hardness_text}. Since we do not have access to the Bayes-optimal predictor, we use the metric
{\small
$
    \mathbb{E}_{x_0 \sim p_{\rm{data}}}\left[\sum_{i=0}^{L-1} \log p_\theta \left( x_0^{\pi(i)} \Big| x_0 [\pi\{i,\ldots,L-1\}] \right) \right]
$}. This captures the accumulation of error across subproblems $p_\theta \left( x_0^{\pi(i)} \Big| x_0 [\pi\{i,\ldots,L-1\}] \right)$, since $p_\theta(x_0 | x_0[M]) = p_{\rm{data}}(x_0 | x_0[M])$ minimizes this metric. Fig.~\ref{fig:scaling_laws} shows a clear gap between different subproblems.

The theoretical and empirical evidence demonstrates that MDMs perform better in estimating $p_{\theta}(x_0 | x_0[M])$ for some subproblems $M$ than for others. We therefore want to avoid encountering hard subproblems $M$ at inference time. In the next section, we show that while vanilla MDM inference can run into such subproblems, simple modifications at the inference stage can effectively circumvent these issues, resulting in dramatic, \emph{training-free} performance improvements.

\section{MDMs can plan around hard problems} \label{sec:inference}
We previously argued that due to the complex nature of masking subproblems, MDM must perform poorly on certain ones $p_\theta(x^i | x_t)$. Therefore, during vanilla MDM inference,
MDM inevitably encounters such difficult subproblems at Step (b). While this might suggest that we need to fundamentally revisit how MDMs are trained, in this section we show that, surprisingly, simple modifications at the inference stage—\emph{without any further training}—can sidestep these issues and lead to significant performance improvements.

\paragraph{MDM offers multiple sampling paths.}
The vanilla MDM inference (Algorithm~\ref{alg:random_sampler_redefine}) aim to align the intermediate distributions with the forward process, as used in continuous diffusion. However, unlike continuous diffusion, the reverse process of MDM allows multiple valid sampling paths (different orders of unmasking the tokens) that match the starting distribution of the forward process of MDM. 

We first show that when we have an ideal MDM that perfectly solves all masking problems, i.e., $p_\theta(x_0^i | x_0[M]) = p_{\rm{data}}(x_0^i | x_0[M])$, then using any sampling path (unmasking the tokens in any order) results in the same distribution. Consider the following sampler: For every step, $S$ is a set with one index selected agnostically (without following any distribution). For any clean sample $x_0$ generated by this sampler, note that $p_\theta(x_0) = \prod_{i=0}^{L-1}  p_\theta \left( x_0^{\pi(i)} \Big| x_0 [\pi\{i,\ldots,L-1\}] \right)$ by chain rule, and this is equal to $\prod_{i=0}^{L-1}  p_{\rm{data}} \left( x_0^{\pi(i)} \Big| x_0 [\pi\{i,\ldots,L-1\}] \right) = p_{\rm{data}}(x_0)$.
Therefore, other choices of $S$, not necessarily following Algorithm~\ref{alg:random_sampler_redefine}, still capture the true likelihood.

\begin{figure}[t]
    \centering
    \includegraphics[width=0.9\linewidth]{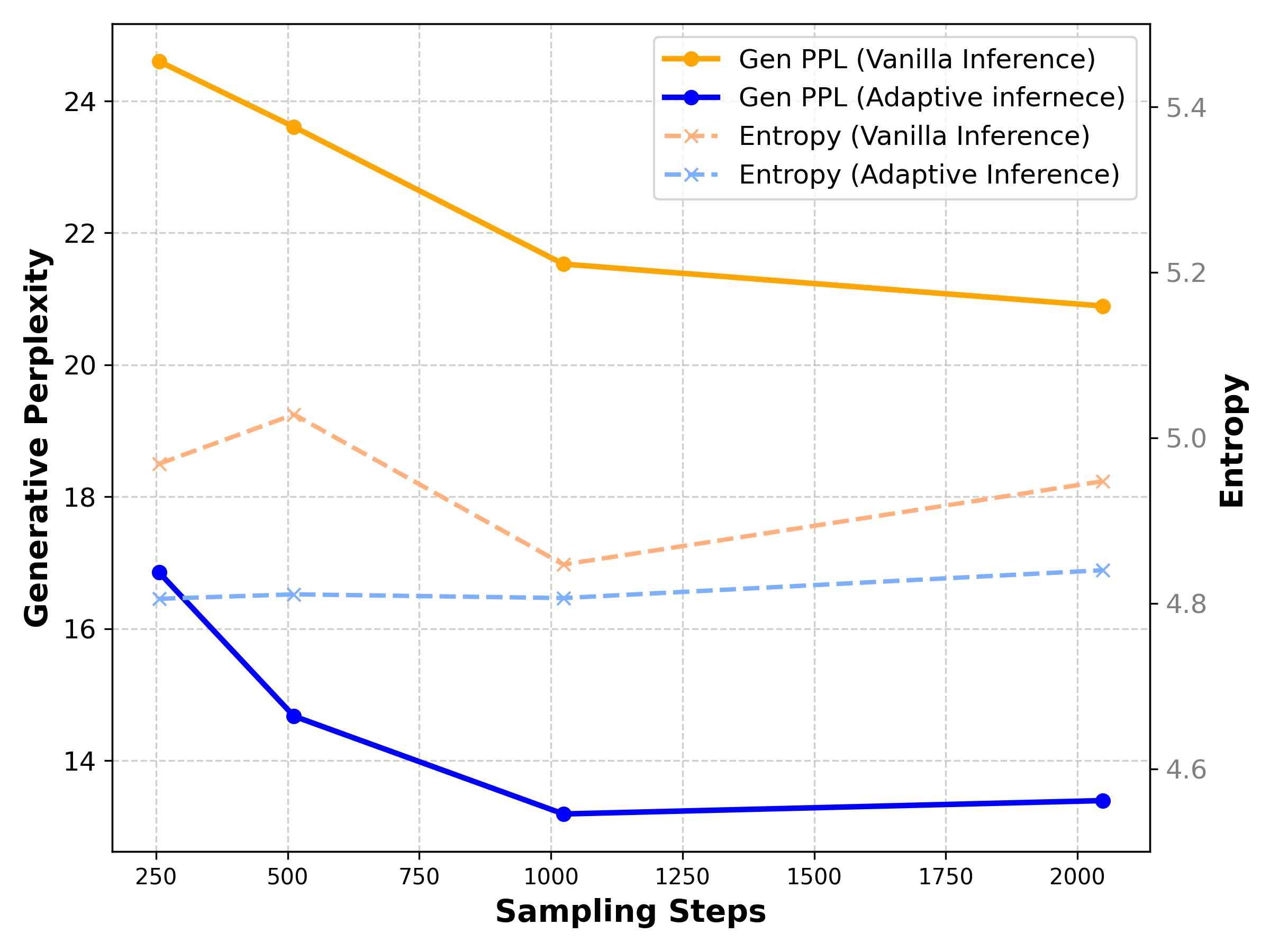}
    \vspace{-0.1in}
    \caption{\textbf{Generative Perplexity.} We compare the resulting generative perplexity (GenPPL) of adaptive vs. vanilla MDM inference. We employ a pretrained $170$M MDM and LLaMA-$7$B \cite{touvron2023llama2} as inference and evaluation, respectively. Adaptive MDM inference {(\color{blue} Blue)} leads to a substantial reduction in generative perplexity, while maintaining the entropy.}
    \label{fig:genppl}
\end{figure}

In practice, unlike this ideal case, MDM does not perform equally well on all subproblems, as shown in Section~\ref{sec:imbalance_error}. Consequently, different sampling paths result in varying likelihood modeling abilities. Motivated by this observation, we consider \emph{adaptive inference for MDMs}:

\vspace{0.05in}
\begin{theo}[Adaptive MDM inference]{alg:sampling_adaptive}
\vspace{-0.1in}
\begin{itemize}
    \item[(a)] Sample a set of masked tokens $\mathcal{S} = \mathcal{F}\left( \theta, x_t \right) \subseteq \{i \mid x_t^i = \mask\}$.
    \item[(b)] For each $i \in \mathcal{S}$, sample $x_s^i \sim p_\theta(x^i | x_t)$.
\end{itemize}
\end{theo}

Instead of selecting $S$ randomly, adaptive MDM inference leverages an oracle $\mathcal{F}(\theta, x_t)$ to select $S$ strategically to avoid hard masking problems. This naturally raises the question of how to design an effective oracle $\mathcal{F}$.

In the following sections, we demonstrate that adaptive MDM inference with careful choices of $\mathcal{F}$ enhance MDM's likelihood matching ability. In other words, a pretrained MDM, even if it performs poorly on certain hard subproblems, \emph{still contains sufficient information to avoid them} when paired with an effective oracle $\mathcal{F}$.

\subsection{Effective design of ordering oracle}
\label{subsec:effective-design}
We introduce two different oracles, Top probability and Top probability margin. Intuitively, both strategies are based on the idea that $S$ should be selected based on how ``certain'' the model is about each position. We caution that these strategies should not be confused with notions like nucleus sampling in ARMs~\cite{holtzman2019curious}; the oracles we describe are for selecting the \emph{position} of the next token to decode, rather than the \emph{value}, and thus are only meaningful in the context of MDMs.

\begin{table}[h]
    \centering 
    \caption{\textbf{L\&O-NAE-SAT}. Adaptive MDM inference achieves better likelihood matching than vanilla MDM inference. Note that naive guessing leads to $75\%$ accuracy, indicating that vanilla inference performs similarly or worse than naive guessing.}
    \vspace{0.1in}
    \begin{tabular}{c c c}
        \toprule
        \textbf{($N,P$)} & \textbf{Vanilla inference} & \textbf{Adaptive inference}\\
        \midrule
        $(25,275)$  & 78.06\%   & 93.76\%  \\
        $(30,270)$  & 75.70\% & 93.54\% \\
        $(40,260)$  & 74.60\%   & 92.21\%  \\
        $(50,250)$  & 67.94\% & 90.01\% \\
        $(100,200)$ & 62.84\% & 88.91\% \\
        \bottomrule
    \end{tabular}
    \label{tab:csp_sampler}
\end{table}

\paragraph{Top probability~\cite{zheng2024reparameterized}.} Suppose we want to unmask $K$ positions at time step $t$, i.e., select $|S|=K$. In the top probability, the uncertainty of a position is estimated by the maximum probability assigned to any value in the vocabulary. More precisely, the certainty at position $i$ is $\max_{j \in \{ 0, \ldots, m-1 \} } p_\theta(x^i = j | x_t)$ and \(\mathcal{F}(\theta, x_t) = \text{Top } K \left(\max p_\theta(x^i | x_t) \right)\). 

Top probability strategy is a good proxy for many tasks and works well in practice \cite{zheng2024reparameterized,ye2024beyond,wang2024diffusion}. However, this approach can often provide misleading estimates of uncertainty. Consider when an MDM is confused between two token values, thus assigning them almost equal but high probabilities.  In this case, unmasking according to top probability may still choose to unmask this position, despite its uncertainty. To mitigate this issue, we propose the following alternative strategy.

\paragraph{Top probability margin.} In this strategy, the uncertainty of a position is instead estimated using the absolute difference between the two most probable values at position $i$. More precisely, if $j_1$ and $j_2$ are the two most probable values in vocabulary according to $p_\theta(x^i | x_t)$ in position $i$, the certainty in the position is given by $| p_\theta(x^i = j_1 | x_t) - p_\theta(x^i = j_2 | x_t) |$  and \(\mathcal{F}(\theta, x_t) = \text{Top } K \left(| p_\theta(x^i = j_1 | x_t) - p_\theta(x^i = j_2 | x_t) | \right)\). When multiple values have similar probabilities at a position, top probability margin strategy will provide a better estimate of the uncertainty of a position, and when there is a single best choice of value then top probability and top probability margin work similarly.

\begin{table}[t]
    \centering
    \caption{Comparison of accuracy for solving the Sudoku puzzle.}
    \vspace{0.1in}
    \begin{tabular}{l >{\centering\arraybackslash}p{1.4cm} c}  % Proper center alignment
        \toprule
        \textbf{Method} & \textbf{\# Param} & \textbf{Accuracy} \\
        \midrule
        ARM (w/o ordering) & \multirow{2}{*}{42M} & 9.73\% \\
        ARM (with ordering) &  & 87.18\% \\
        \midrule
        MDM (vanilla) & \multirow{3}{*}{6M} & 6.88\% \\
        MDM (Top probability) &  & 18.51\% \\
        MDM (Top prob. margin) &  & 89.49\% \\
        \bottomrule
    \end{tabular}
    \label{tab:sudoku-results}
\end{table}

\subsection{Adaptive MDM inference} \label{subsec:adaptive_inference}

In this section, we experimentally validate that adaptive MDM inference helps MDMs avoid hard subproblems, leading to better likelihood matching. We first show our results on L\&O-NAE-SAT and text data, before turning to our primary application to logic puzzles.

\textbf{L\&O-NAE-SAT and text data.} For the L\&O-NAE-SAT distribution defined in \cref{sec:imbalance_error}, we evaluate the effectiveness of adaptive inference by measuring the accuracy in predicting the observation tokens. Table~\ref{tab:csp_sampler} in the appendix reveals a clear improvement over vanilla inference. For the text dataset, we evaluate using the standard metric of \emph{generative perplexity}, by which likelihood is measured by a large language model. We also compute the entropy of the generated samples to ensure both inference strategies exhibit similar levels of diversity. As shown in Fig.~\ref{fig:genppl}, we observe a substantial decrease in generative perplexity using adaptive inference. We defer further experimental details to Appendix~\ref{appendix:exp_detail_inference}.

\textbf{Logic puzzles.} We consider two different types of logic puzzles: Sudoku and Zebra (Einstein) puzzles. Intuitively, for Sudoku, some empty (masked) cells are significantly easier to predict than others and we want to choose the cells that are easier to predict during the inference. We evaluate the effectiveness of adaptive MDM inference over vanilla MDM inference in selecting such cells.\footnote{A prior work \cite{ye2024beyond} reported that a $6$M MDM with \topk inference achieves 100\% accuracy on Sudoku. Given that a 6M MDM with \topk only achieves 18.51\% on our dataset (Table~\ref{tab:sudoku-results}), this suggests that the Sudoku dataset in~\cite{ye2024beyond} is significantly easier than ours.}

To measure the performance of an inference method, we use the percentage of correctly solved puzzles. For both puzzles, we use train and test datasets from \cite{shah2024causal}. For the Sudoku puzzle (Table~\ref{tab:sudoku-results}) we observe that adaptive MDM inference, in particular, Top probability margin strategy, obtains substantially higher accuracy (89.49\%) compared to vanilla MDM inference (6.88\%). Additionally, Top probability margin obtains higher accuracy (89.49\%) than Top probability strategy (18.51\%). As mentioned in \cref{subsec:effective-design}, this is because Top probability margin strategy more reliably estimates uncertainty when multiple competing values are close in probability at a given position, as is often the case in Sudoku. For the Zebra puzzle, as shown in \Cref{tab:zebra-results}, we observe a consistent result: Top probability (98.5\%) and Top probability margin (98.3\%) outperform vanilla MDM inference (76.9\%).

\begin{table}[h]
    \centering
    \caption{Comparison of accuracy for solving the Zebra puzzle.}
    \vspace{0.1in}
    \begin{tabular}{l >{\centering\arraybackslash}p{1.4cm} c}  % Proper center alignment
            \toprule
        \textbf{Method} & \textbf{\# Param} & \textbf{Accuracy} \\
        \midrule
        ARM (w/o ordering) & \multirow{2}{*}{42M} & 80.31 \% \\
        ARM (with ordering) &  & 91.17 \% \\
        \midrule
        MDM (vanilla) & \multirow{3}{*}{19M} & 76.9 \% \\
        MDM (Top probability) &  & 98.5 \% \\
        MDM (Top prob. margin) &  & 98.3 \% \\
        \bottomrule
    \end{tabular}
    \label{tab:zebra-results}
\end{table}

\subsection{Eliciting sequence-dependent reasoning paths using adaptive MDM inference in logic puzzles} 
\label{sec:sequence-dependent-tasks}

In this section, we study the effectiveness of adaptive MDM inference in finding the right reasoning/generation order for tasks where every sequence has a different ``natural'' order. To do so, we will compare the performance of adaptive MDM inference to that of ARM on Sudoku and Zebra puzzles. For these puzzles, the natural order of generation is not only different from left-to-right, but it is also sequence-dependent. For such tasks, prior works have shown that ARMs struggle if the information about the order is not provided during the training \cite{shah2024causal, lehnert2024beyond}. Therefore, to obtain a strong baseline, we not only consider an ARM trained without the order information but also consider an ARM trained with the order information for each sequence in the training data. Note that the latter is a much stronger baseline than the former as one can hope to teach the model to figure out the correct order by some form of supervised teacher forcing (as performed in \citet{shah2024causal, lehnert2024beyond}), eliminating the issue of finding the right order in an unsupervised manner. 

We compare ARMs and MDMs for Sudoku in \cref{tab:sudoku-results} and Zebra puzzles in \cref{tab:zebra-results}. We observe that for both, Top probability margin-based adaptive MDM inference not only outperforms the ARM trained without ordering information, but it \emph{\textbf{even outperforms the ARM trained with ordering information}}! This shows that the \emph{unsupervised} way of finding the correct order and solving such logic puzzles using adaptive MDM inference outperforms the \emph{supervised} way of finding the correct order and solving such puzzles using an ARM, and is significantly less computationally intensive.

\subsection{Adaptive MDM inference on natural language tasks}

To examine the effect of different inference strategies on text benchmarks, we adapted LLaDA, the 8B MDM model from \cite{nie2025large}. We compare three inference strategies: vanilla, top probability, and top probability margin. The results are presented in Table~\ref{tab:sampler_performance}.

We see that both adaptive MDM inference strategies, top probability and top probability margin, consistently outperform vanilla MDM inference. Notably, top probability margin demonstrates a clear advantage over top probability in challenging tasks like HumanEval-Multiline (infill), HumanEval-Split Line (infill), and Math. This is because Top probability margin provides a more reliable estimate of uncertainty when multiple tokens have similar probabilities, a frequent occurrence in these difficult tasks. These results further underscore the potential for developing new, sophisticated adaptive inference strategies for various tasks. We provide experimental details in Appendix~\ref{sec:llada_detail}.

\begin{table*}[!htbp]
    \centering
    \caption{Performance of different inference strategies for LLaDa 8B model on coding and math tasks.}
    \vspace{0.05in}
    \label{tab:sampler_performance}
    \begin{tabular}{lcccccc}
        \toprule
        \textbf{Method} & \textbf{HumanEval-Single} & \textbf{HumanEval-Multi} & \textbf{HumanEval-Split} & \textbf{Math} & \textbf{MMLU} & \textbf{ROCStories} \\
        \midrule
        Vanilla & 31.8\% & 16.5\% & 14.2\% & 28.5\% & 33.2\% & 21.23\% \\
        Top probability & 32.9\% & 20.8\% & 18.4\% & 31.3\% & \textbf{36.5\%} & 21.10\% \\
        Top prob. margin & \textbf{33.5\%} & \textbf{25.4\%} & \textbf{22.3\%} & \textbf{34.3\%} & 35.4\% & \textbf{21.41\%} \\
        \bottomrule
    \end{tabular}
\end{table*}

\subsection{Easy to hard generalization}

In the previous section we showed that when the training and inference sequences come from the same distribution, order-agnostic training of MDMs combined with adaptive inference can perform very well on logic puzzles. To evaluate if the model has learned the correct way of solving the puzzles and test the robustness of adaptive inference, we also test the MDMs on harder puzzles than the ones from training, for Sudoku. 

We keep the training dataset the same as proposed in \citet{shah2024causal}. \citet{shah2024causal} created this dataset from \citet{david_g__radcliffe_2020} by selecting the puzzles that can be solved using 7 fixed strategies and do not require backtracking-based search. We use the remaining puzzles in \citet{david_g__radcliffe_2020} as our hard dataset. Hence, these puzzles all use a strategy not seen during training and/or backtracking to obtain the correct solution.  

\begin{table}[t]
    \centering
    \caption{Comparison of accuracy for solving the hard Sudokus.}
    \vspace{0.1in}
    \begin{tabular}{l >{\centering\arraybackslash}p{1.4cm} c}
    \toprule
        \textbf{Method} & \textbf{$\#$Param} & \textbf{Accuracy} \\
        \midrule
         ARM (with ordering) & 42M & 32.57 \% \\
         \midrule
         MDM (random) & \multirow{3}{*}{6M} & 3.62 \% \\
         MDM (Top probability) & & 9.44 \% \\
         MDM (Top prob. margin) &  & 49.88 \% \\
         \bottomrule
    \end{tabular}
    \label{tab:easy-to-hard-sudoku}
\end{table}

We measure the accuracy of MDMs and ARMs on the hard test set and present the results in \cref{tab:easy-to-hard-sudoku}. We see that the Top probability margin-based adaptive MDM inference strategy (49.88\%) again significantly outperforms ARMs trained with order information (32.57\%). In particular, although the accuracy drops for both methods due to the more challenging test set, MDMs with adaptive inference appear to be more robust to this distribution shift than ARMs. We believe this is due to the fact that MDMs try to solve a significantly higher number of infilling problems than ARMs ($\exp(L)$ compared to $L$) and therefore are able to extract knowledge about the problem more efficiently than ARMs. 

\section{Conclusion}

In this work, we examined the impact of token generation order on training and inference in MDMs. We provided theoretical and experimental evidence that MDMs train on hard masking problems. We also demonstrated that adaptive inference strategies can be used to sidestep these hard problems. For logic puzzles, we find that this leads to dramatic improvements in performance not just over vanilla MDMs, but even over ARMs trained with teacher forcing to learn the right order of decoding. An important direction for future work is to go beyond the relatively simple adaptive strategies to find a better generation order like top probability and top probability margin considered here.

\paragraph{Acknowledgements.} JK thanks Kiwhan Song for discussions about MDM training. KS and VK are supported by the NSF AI Institute for Foundations of Machine Learning (IFML). KS and VK thank the computing support on the Vista GPU Cluster through the Center for Generative AI (CGAI) and the Texas Advanced Computing Center (TACC) at UT Austin. KS thanks Nishanth Dikkala for the initial discussions about the project. SK acknowledges: this work has been made possible in part by a
gift from the Chan Zuckerberg Initiative Foundation to establish the Kempner Institute
for the Study of Natural and Artificial Intelligence and support from the
Office of Naval Research under award N00014-22-1-2377. SC is supported by the Harvard Dean's Competitive Fund for Promising Scholarship and thanks Brice Huang and Sidhanth Mohanty for enlightening discussions about computational-statistical tradeoffs for planted CSPs. 

\section*{Impact statement}
This paper advances the understanding of discrete diffusion models, contributing to the broader field of Machine Learning. There are many potential societal consequences of our work, none of which we feel must be specifically highlighted here.

\bibliography{main}
\bibliographystyle{icml2025}

\newpage
\appendix
\onecolumn
\input{appendix}
\end{document}

%% file: appendix.tex
\section{Related works}
\paragraph{Discrete diffusion models.}
(Continuous) diffusion models were originally built on continuous-space Markov chains with Gaussian transition kernels \cite{sohldickstein2015deep,ho2020denoising}. This was later extended to continuous time through the theory of stochastic differential equations \cite{song2021score}. In a similar vein, discrete diffusion models have emerged from discrete-space Markov chains \cite{hoogeboom2021argmax}. Specifically, \cite{austin2023structured} introduced D3PM with various types of transition matrices. Later, \citet{lou2024discrete} proposed SEDD, incorporating a theoretically and practically robust score-entropy objective. Additionally, \citet{varma2024glauber,liu2024think} introduced novel modeling strategies that classify tokens in a noisy sequence as either signal (coming from clean data) or noise (arising from the forward process). In particular, \citet{liu2024think} uses this to give a \emph{planner} that adaptively determines which tokens to denoise. While this is similar in spirit to our general discussion about devising adaptive inference strategies, we emphasize that their approach is specific to discrete diffusions for which the forward process \emph{scrambles} the token values, rather than masking them.

\paragraph{Masked diffusion models.} Meanwhile, the absorbing transition kernel has gained popularity as a common choice due to its better performance than other kernels. Building on this, \citet{sahoo2024simple,shi2025simplified} aligned its framework with continuous diffusion, resulting in a simple and principled training recipe, referring to it as \emph{Masked Diffusion Model}. Subsequent studies have explored various aspects of MDM. \citet{gong2024scaling} efficiently trained MDM via adaptation from autoregressive models, scaling MDM up to 7B parameters. \citet{zheng2024maskeddiffusionmodelssecretly} interpreted 
MDMs as order-agnostic learners and proposed a first-hitting sampler based on this insight. \citet{ye2024beyond,gong2024scaling} demonstrated that MDM outperforms autoregressive models in reasoning and planning tasks, emphasizing its impact on downstream applications. \citet{nie2024scaling} examined the scaling laws of MDM, while \citet{xu2024energy,liu2024copula} identified limitations in capturing coordinate dependencies when the number of sampling steps is small and proposed additional modeling strategies to address this issue. \citet{schiff2024simple} studied conditional generation using MDM and \citet{rectorbrooks2024steering} tackled the challenge of controlling generated data distributions through steering methodologies. \citet{chen2024convergence} provided a theoretical analysis showing that sampling error is small given accurate score function estimation.

\paragraph{Any-order reasoning.} Even though language tasks generally have a natural order of ``left-to-right" token generation, in many tasks like planning, reasoning, and combinatorial optimization, the natural order of token generation can be quite different from ``left-to-right". Even though prominent autoregressive-based language models achieve impressive performance on various tasks, many works \cite{golovneva2024reverse, chen2024premise, kitouni2024factorization} have shown that this performance is tied to the training order of the tasks and therefore can cause brittleness from it. For example, \citet{chen2024premise} showed that simply permuting the premise order on math tasks causes a performance drop of 30\%. The reason behind such brittleness regarding the ordering is the inherent ``left-to-right" nature of the autoregressive models. Several works \cite{liao-etal-2020-probabilistically} have tried to address this issue in the autoregressive framework. In particular, \cite{papadopoulos2024arrows} highlighted the significance of left-to-right ordering in natural language by comparing its likelihood to that of the reverse (right-to-left) ordering.

Recently, discrete diffusion models have emerged as a promising approach for discrete data apart from autoregressive models. Additionally, the order-agnostic training of discrete diffusion models opens up the multiple sampling paths during the inference but it also faces some challenges during the training therefore, they seem a promising approach to elicit any order reasoning. \citet{zheng2024reparameterized} proposed different ways of implementing an adaptive inference strategy for MDM but a \emph{concrete understanding of why such an adaptive inference strategy is needed is still lacking}. In this work, we explore various aspects of vanilla MDM training and how adaptive MDM inference can mitigate the issues raised by vanilla MDM training and elicit any order reasoning. 

We also want to mention the concurrent work by \citet{peng2025path} that proposes an alternative adaptive inference strategy by selecting $\mathcal F(\theta, x_t)$ based on the BERT model or the denoiser itself. In particular, \citet{peng2025path} uses the BERT model or the denoiser to obtain the uncertainty of a token and then uses Top-$K$ to decide the positions to unmask it. In contrast to their work, we disentangle the impact of token ordering on MDM training vs. MDM inference and provide a more complete understanding of the motivations for and benefits of adaptive inference. Additionally, our results indicate drawbacks to using Top-$K$ strategy as opposed to Top-$K$ margin in deciding which tokens to unmask when there are multiple values with high probabilities.

\paragraph{Beyond autoregressive models.}
Efforts to learn the natural language using non-autoregressive modeling began with BERT \cite{devlin-etal-2019-bert}. Non-causal approaches can take advantage of the understanding the text data representation. \cite{chang2022maskgit} adopted a similar approach for learning image representations. Building on these intuitions, 
\cite{shih2022training,hoogeboom2022autoregressive} proposed any-order modeling, which allows a model to generate in any desired order. \citet{shih2022training} made the same observation that any-order models by default have to solve exponentially more masking problems than autoregressive models. However, whereas our work shows that learning in the face of this challenging task diversity can benefit the model at inference time, their work sought to alleviate complexity at training time by reducing the number of masking problems that need to be solved.

\section{Technical details from Section~\ref{sec:hardness}}

\paragraph{Notations.} Throughout this section, we use $x^i$ to denote the $i$-th coordinate of the vector $x$ and $z{(j)}$ to denote the $j$-th example. The $i$-th coordinate of the vector $z{(j)}$ is denoted by $z{(j)}^i$.

\subsection{Additional example: sparse parity observations}
\label{app:parity}

\begin{example}[Noisy sparse parity observations]\label{example:xor}
    Let $m = 2$,  $k\in\mathbb{N}$, and $N^2\log N \ll P \le N^{0.49k}$. Fix \emph{noise rate} $\eta > 0$ as well as strings $z{(1)},\ldots, z{(P)}$ sampled independently and uniformly at random from the set of $k$-sparse strings in $\{0,1\}^N$. For each $j\in[P]$, define $\mathcal{O}_j(x)$ to be the distribution which places mass $1 -\eta$ on $1$ (resp. $2$) and mass $\eta$ on $2$ (resp. $1$) if $\sum_i x^i z{(j)}^i$ is odd (resp. even). Note that for $k = O(1)$, each of these observations is efficiently learnable by brute-force.
\end{example}

Below we show that for a certain range of masking fractions, a constant fraction of the masking problems for the corresponding L\&O distributions are computationally hard under the \emph{Sparse Learning Parity with Noise} assumption~\cite{alekhnovich2003more}. Formally we have:

\begin{proposition}\label{prop:lpn}
    Let $0 < \alpha < 1$ be an arbitrary absolute constant, and let $\eta = 1/\mathrm{poly}(N)$ be sufficiently large. Let $x$ be a sample from a L\&O distribution $p_{\rm data}$ with noisy parity observations as defined in Example~\ref{example:xor}. Suppose each token is independently masked with probability $\alpha$, and $M$ is the set of indices for the masked tokens. If $1 - 1/N \le \alpha \le 1 - 1/2N$, then under the Sparse Learning Parity with Noise (SLPN) assumption (see Definition~\ref{def:lpn}), with constant probability over $M$, no polynomial-time algorithm can solve the resulting masking problem of predicting any of the masked tokens among $x^{\pi(1)}, \ldots, x^{\pi(N)}$ given $x[M]$.
\end{proposition}

We note that it is important for us to take the observations to be \emph{sparse} parities and to leverage the \emph{Sparse} Learning Parity with Noise assumption. If instead we used \emph{dense} parities and invoked the \emph{standard} Learning Parity with Noise (LPN) assumption, we would still get the hardness of masking problems, but the observations themselves would be hard to learn, assuming LPN. This result is based on the following standard hardness assumption:

\begin{definition}[Sparse Learning Parity with Noise]\label{def:lpn}
    Given input dimension $N$, noise parameter $0 < \eta < 1/2$, and sample size $P$, an instance of the \emph{Sparse Learning Parity with Noise (SLPN)} problem is generated as follows:
    \begin{itemize}
        \item Nature samples a random bitstring $x$ from $\{0,1\}^N$
        \item We observe $P$ examples of the form $(x{(i)},y{(i)})$ where $x{(i)}$ is sampled independently and uniformly at random from $k$-sparse bitstrings in $\{0,1\}^N$, and $y$ is given by $\epsilon_i + \langle x{(i)}, x\rangle \pmod{2}$, where $\epsilon_i$ is $1$ with probability $\eta$ and $0$ otherwise.
    \end{itemize}
    Given the examples $\{(x{(i)},y{(i)})\}^P_{i=1}$, the goal is to recover $x$.

    The \emph{SLPN assumption} is that for any $P = N^{(1 - \rho)k/2}$ for constant $0 < \rho < 1$, and any sufficiently large inverse polynomial noise rate $\eta$, no $\mathrm{poly}(N)$-time algorithm can recover $x$ with high probability.
\end{definition}

\begin{proof}[Proof of Proposition~\ref{prop:lpn}]
    With probability at least $1 - (1 - 1/N)^N \ge \Omega(1)$, all of the variable tokens $x^{\pi(i)}$ for $i \le N$ are masked. Independently, the number of unmasked tokens among the observation tokens $\mathcal{O}_j$ is distributed as $\mathrm{Bin}(P, 1-\alpha)$, so by a Chernoff bound, with probability at least $1 - e^{-\Omega(P/N^2)} = 1 - 1/\mathrm{poly}(N)$ we have that at least $P/4N = \Omega(N\log N)$ observation tokens are unmasked. The masking problem in this case amounts to an instance of SLPN with input dimension $N$ and sample size in $[\Omega(N\log N), O(N^{0.49k})]$. Because of the lower bound on the sample size, prediction of $\mathbf{x}^M$ is information-theoretically possible. Because of the upper bound on the sample size, the SLPN assumption makes it computationally hard. As a result, estimating the posterior mean on any entry of $\mathbf{x}^M$ given the unmasked tokens is computationally hard as claimed.
\end{proof}

\subsection{Additional example: random slab observations}
\label{app:slab}

\begin{example}[Random slab observations]\label{example:perceptron}
    Let $m = 2$ and $P = \gamma N^2$ for constant $\gamma > 0$. Fix \emph{slab width} $\beta$ and vectors $z{(1)}, \ldots, z{(P)}$ sampled independently from $\mathcal{N}(0,I)$. For each $j\in[P]$, define the corresponding observation $\mathcal{O}_j(x)$ to be deterministically $1$ if $|\langle z{(j)}, 2x - \mathbf{1}\rangle| \le \beta\sqrt{N}$, and deterministically $0$ otherwise.
\end{example}

In~\cite{alaoui2024hardness}, it was shown that \emph{stable} algorithms (Definition~\ref{def:stable}), which encompass many powerful methods for statistical inference like low-degree polynomial estimators, MCMC, and algorithmic stochastic localization~\cite{gamarnik2021overlap}, are unable to sample from the posterior distribution over a random bitstring conditioned on it satisfying $|\langle z{(j)}, x\rangle| \le \beta\sqrt{N}$ for any $\Theta(N)$ number of constraints $z{(1)},\ldots,z{(P')}$, provided $P'$ is not too large that the support of the posterior is empty. This ensemble is the well-studied \emph{symmetric perceptron}~\cite{aubin2019storage}. The following is a direct reinterpretation of the result of~\cite{alaoui2024hardness}:

\begin{proposition}\label{prop:perceptron_masking}
    Let $p_{\rm data}$ be a L\&O distribution with random slab observations as defined in Example~\ref{example:perceptron}, with parameter $\gamma > 0$ and slab width $\beta > 0$. There exists a constant $c_\beta > 0$ such that for any absolute constant $0 < c < c_\beta$, 
    if $1 - c_\beta N/2P \le \alpha \le 1 - c N / P$ and $\gamma > c_\beta$, the following holds. Let $p'_{\rm data}$ denote the distribution given by independently masking every coordinate in $p_{\rm data}$ with probability $\alpha$. Then \emph{any} $(1 - \tilde{\Omega}(1/\sqrt{N}))$-stable algorithm, even one not based on masked diffusion, which takes as input a sample $x'$ from $p'_{\rm data}$ and, with probability $1 - o(1)$ outputs a Wasserstein-approximate\footnote{Here the notion of approximation is $o(1)$-closeness in Wasserstein-2 distance.} sample from $p_{\rm data}$ conditioned on the unmasked tokens in $x'$, must run in super-polynomial time.
\end{proposition}

The upshot of this is that any stable, polynomial-time masked diffusion sampler will, with non-negligible probability, encounter a computationally hard masking problem at some point during the reverse process.

For the proof, we first formally define the (planted) symmetric Ising perceptron model:

\begin{definition}
    Let $\alpha, \beta > 0$. The \emph{planted symmetric Ising perceptron} model is defined as follows: 
    \begin{itemize}
        \item Nature samples $\sigma$ uniformly at random from $\{\pm 1\}^N$
        \item For each $j = 1,\ldots,P = \lfloor \alpha N\rfloor$, we sample  $z{(j)}$ independently from $\mathcal{N}(0,I_N)$ conditioned on satisfying $|\langle z{(j)}, \sigma\rangle| \le \beta\sqrt{N}$.
    \end{itemize}
    The goal is to sample from the posterior on $\sigma$ conditioned on these observations $\{z{(i)}\}^P_{i=1}$.
\end{definition}

Next, we formalize the notion of \emph{stable algorithms}.

\begin{definition}\label{def:stable}
    Given a matrix $Z\sim\mathcal{N}(0,1)^{\otimes P\times N}$, define $Z_t = tZ + \sqrt{1 - t^2}Z'$ for independent $Z'\sim\mathcal{N}(0,1)^{\otimes P\times N}$. 
    A randomized algorithm $\mathcal{A}$ which takes as input $Z\in\mathbb{R}^{P\times N}$ and outputs an element of $\{\pm 1\}^N$ is said to be \emph{$t_N$-stable} if $\lim_{N\to\infty} W_2(\mathrm{law}(\mathcal{A}(Z)), \mathrm{law}(\mathcal{A}(Z_t))) = 0$.
\end{definition}

As discussed at depth in~\cite{gamarnik2021overlap}, many algorithms like low-degree polynomial estimators and Langevin dynamics are stable. 

\begin{theorem}[Theorem 2.1 in~\cite{alaoui2024hardness}\footnote{Note that while the theorem statement in~\cite{alaoui2024hardness} refers to the non-planted version of the symmetric binary perceptron, the first step in their proof is to argue that these two models are mutually contiguous in the regime of interest.}]\label{thm:gamarnik}
    For any constant $\beta > 0$, there exists $c_\beta > 0$ such that the following holds for all constants $0 < \alpha < c_\beta$. For $t_N \le 1 - \Omega(\log^2(n) / n^2)$, any $t_N$-stable randomized algorithm $\mathcal{A}$ which takes as input $Z = (z{(1)},\ldots,z{(P)})$ and outputs an element of $\{\pm 1\}^N$ will fail to sample from the posterior on $\sigma$ conditioned on $Z$ in the symmetric Ising perceptron model to Wasserstein error $o(\sqrt{N})$.
\end{theorem}
% \kulin{define stable algorithm.}

\begin{proof}[Proof of Proposition~\ref{prop:perceptron_masking}]
    By a union bound, with probability at least $1 - (1 - \alpha) N \ge 1 - c_\beta N^2/P \ge 1 - c_\beta /\gamma$ over a draw $x' \sim p'_{\rm data}$, all of the $x^{\pi(i)}$ tokens are masked. The number of unmasked tokens in $x'$ among the observations $\mathcal{O}_j$ is distributed as $\mathrm{Bin}(P, 1 - \alpha)$. By a Chernoff bound, this is in $[3cN/4, 3c_\beta N/4]$ with at least constant probability. The claim then follows immediately from Theorem~\ref{thm:gamarnik} above. 
\end{proof}

\subsection{Proof outline of Proposition~\ref{prop:csp}} \label{appendix:pf_outline_hardness}
 To understand the proof idea, we consider the case where all the latent tokens are masked and some of the observation tokens are unmasked. In this case, the prediction task reduces to learning to recover the latent tokens that are consistent with the observations. Intuitively, each observation provides some constraints and the task is to recover an assignment that satisfies the constraints. This is reminiscent of \emph{Constraint Satisfaction Problems} (CSPs). Indeed, to show the hardness result, we use the rich theory developed for \emph{planted} CSPs at the intersection of statistical physics and average-case complexity. 

In a planted CSP, there is an unknown randomly sampled vector $y$ of length $N$ and, one is given randomly chosen Boolean constraints %(e.g., $y^7 \wedge y^{8} \wedge y^{3} = 1$) 
which $y$ is promised to satisfy, and the goal is to recover $y$ as best as possible (see Definition~\ref{def:plantedcsp}). Prior works have shown the hardness of efficiently learning to solve the planted CSP problem \cite{krzakala2009hiding, alaoui2024hardness}. We show the hardness of masking problems in L\&O distributions based on these results. Consider the ground truth latent tokens as the random vector $y$ and each observation as a constraint. In this case, the problem of learning to recover the latent tokens from the observation tokens reduces to recovery for the planted CSP.

There are precise predictions for the values of vocabulary size $m$ and the number of observations for which the information-theoretically best possible overlap and the best overlap achievable by any computationally efficient algorithm are different. We show that these predictions directly translate to predictions about when masking problems become computationally intractable:

\begin{figure}
    \centering
    \hspace{-5mm}    \includegraphics[width=0.5\linewidth]{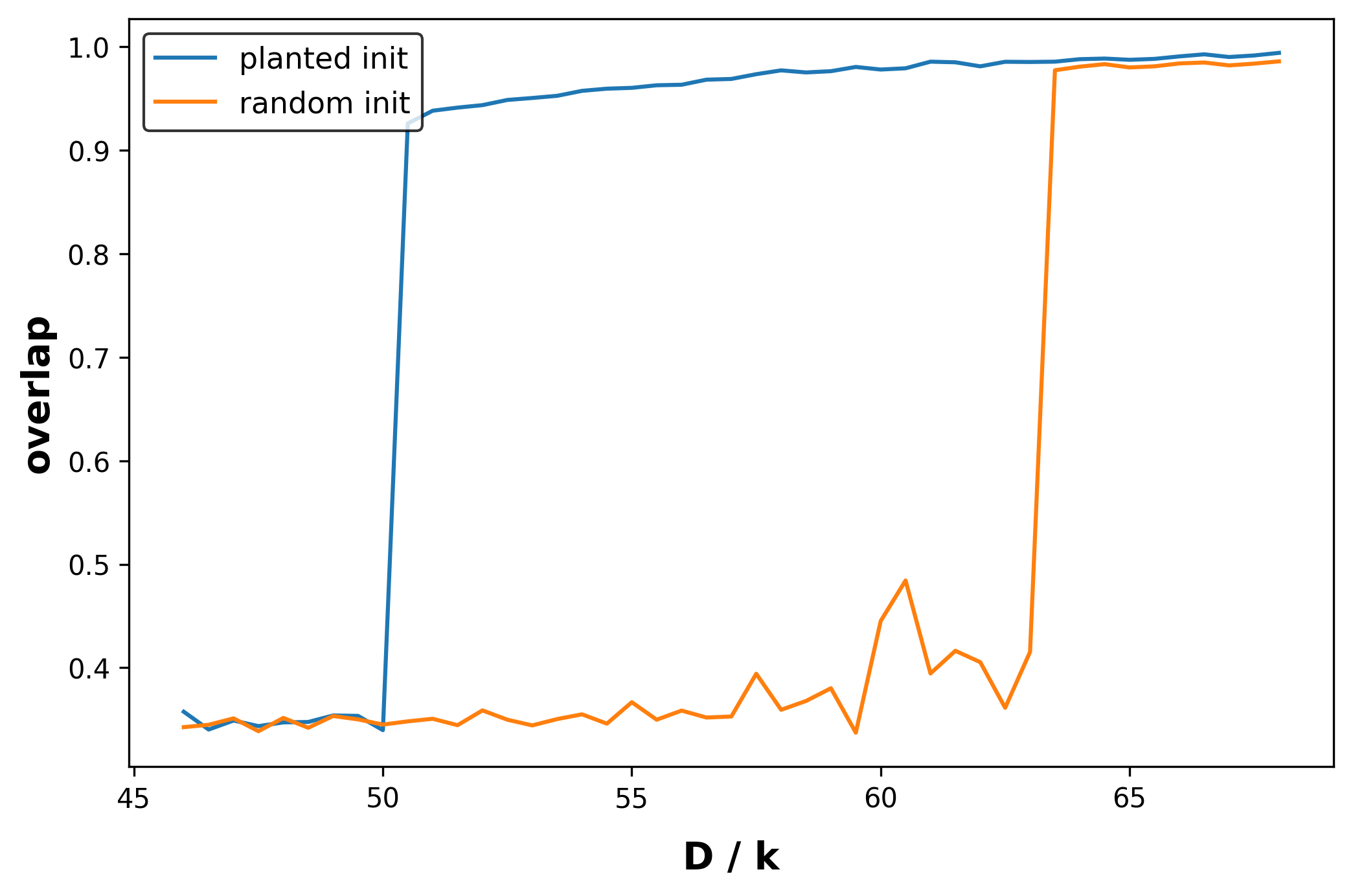}
    \vspace{-0.05in}
    \caption{Overlap achieved by belief propagation initialized at ground truth versus random for planted CSP with $k = 3$, $m = 3$, and $g = \mathrm{NAE}$, for $N = 10000$ and varying choices of average degree $D$. $D_{\rm KS} / K$ can be shown analytically to be $64$, consistent with the phase transition depicted. Plot suggests $D_{\rm cond}/K \approx 50$. By Prop.~\ref{prop:csp} this implies a range of masking fractions at which $\Omega(1)$ fraction of masking problems are computationally hard.}
    \label{fig:csp}
\end{figure}

As a simple example, let us consider sparse predicate observations with $k=2$ and $g(x',x'') = \mathbf{1}[x' \neq x'']$. These can be formally related to the well-studied problem of \emph{planted $m$-coloring}. In the planted $m$-coloring, a random graph of average degree $D$ is sampled consistent with an unknown vertex coloring and the goal is to estimate the coloring as well as possible~\cite{krzakala2009hiding}, as measured by the \emph{overlap} of the output of the algorithm to the ground-truth coloring (see Definition~\ref{def:plantedcsp}). As a corollary of our main result, we show that when all the latent tokens $x^{\pi(1)}, \ldots, x^{\pi(N)}$ are masked and a few unmasked observation tokens provide the information of the form $g(x^{\pi(i)}, x^{\pi(j)}) = \mathbf{1}[ x^{\pi(i)} \neq x^{\pi(j)} ]$ for $i, j \leq N$, then solving the masking problem can be reduced to solving planted coloring.

For planted $m$-coloring, when $m = 5$ the thresholds in Proposition~\ref{prop:csp} are given by $D_{\rm KS} / 2 = 16$ and $D_{\rm cond} / 2 \approx 13.23$~\cite{krzakala2009hiding} (the factor of $2$ here is simply because the observations correspond to \emph{ordered} subsets of size $2$). For general predicates and arities, there is an established recipe for numerically computing $D_{\rm KS}$ and $D_{\rm cond}$ based on the behavior of the \emph{belief propagation} algorithm (see the discussion in Appendix~\ref{app:planted_result}). As an example, in Fig.~\ref{fig:csp}, we execute this recipe for $m = 3$, $k = 3$, and $g$ given by the Not-All-Equal predicate $\mathrm{NAE}(x',x'',x'') = 1 - \mathbf{1}[x' = x'' = x''']$ to obtain thresholds that can be plugged into Proposition~\ref{prop:csp}.

\paragraph{Additional examples of the hardness.} The above setup can also be generalized to capture \emph{Bayesian constraint satisfaction problems}~\cite{montanari2008estimating,liu2022statistical}, one notable example of which is the stochastic block model~\cite{PhysRevE.84.066106}. There are analogous predictions for the onset of hardness of inference, which can likewise be translated to hardness of masking problems for seemingly benign L\&O distributions. In Appendix~\ref{app:parity} and~\ref{app:slab}, we give two more examples of L\&O distributions for which order-aware training is tractable yet order-agnostic training of the MDM is computationally hard. 

First, we consider L\&O distributions whose observations are sparse, noisy parities in the latents and deduce hardness for order-agnostic training from the Sparse Learning Parity with Noise assumption~\cite{alekhnovich2003more}. We then consider L\&O distributions whose observations are \emph{generalized linear models} in the latents, and deduce hardness for a large class of efficient algorithms from existing results on Lipschitz hardness~\cite{alaoui2024hardness} for the symmetric binary perceptron~\cite{aubin2019storage}.

\subsection{Proof of Proposition~\ref{prop:csp}: sparse predicate observations}
\label{app:planted_result}

Here we formally define the relevant notions needed to formalize our claim about hardness in Proposition~\ref{prop:csp}.

\begin{definition}[Planted CSPs]\label{def:plantedcsp}
    Given arity $k\in\mathbb{N}$, vocabulary/alphabet size $m\in\mathbb{N}$, predicate $g: \{1,\ldots,m\}^k \to \{0,1\}$, latent dimension $N$, and clause density $P/N$, the corresponding \emph{planted constraint satisfaction problem} is defined as follows: Nature samples an unknown assignment $\sigma$ uniformly at random from $\{ 1, \ldots, m \}^N$, and then for each ordered $k$-tuple $S$ of distinct elements from $[N]$, we observe the \emph{clause} $S$ independently with probability $\phi / N^{k-1}$ if $g(\sigma|_S) = 1$.

    To measure the quality of an algorithm for recovering $\sigma$ given the observations, define the \emph{overlap} between an estimate $\hat{\sigma}$ and the ground truth $\sigma$ by $d(\sigma,\hat{\sigma}) \triangleq \min_{\pi\in\mathbb{S}_N} \sum_i \mathbf{1}[\sigma_i = \pi(\hat{\sigma}_i)]$ where $\mathbb{S}_N$ denotes the set of all permutations of $\{0, 1, \ldots, N-1\}$. Define the \emph{average degree} to be $kP/N$, i.e. the expected number of variables that share at least one clause with a given variable.
\end{definition}

We begin by defining the central algorithm driving statistical physics predictions about hardness for random constraint satisfaction problems: belief propagation (BP).

\begin{definition}[BP update rules]\label{def:BP}
    Belief propagation is an algorithm that iteratively updates a set of \emph{messages} $\{\msg^{i\to S}_c[t], \msg^{S\to i}_c[t]\}$, where $i, S$ range over all pairs of variable indices $i\in[N]$ and observations $S\ni i$. At time $t+1$, the messages are computed via
    \begin{align}
        \msg^{i\to S}_c[t+1] &\propto \prod_{T: i\in T\neq S} \msg^{T\to i}_c[t] \\
        \msg^{S\to i}_c[t+1] & \propto \sum_{\overline{\sigma}\in \{ 1,\ldots,m \}^{S\backslash i}} g(\overline{\sigma}\cup_i c) \prod_{j: i\neq j\in S} \msg^{j\to S}_{\overline{\sigma}_j}[t]\,,
    \end{align}
    where $\overline{\sigma}\cup_i c \in \{1, \ldots, m \}^S$ assigns $c$ to entry $i$ and $\overline{\sigma}$ to the remaining entries.

    A set of messages can be used to estimate the marginals of the posterior on $\sigma$ conditioned on the observations as follows. The marginal on the $i$-th variable has probability mass function over $\{1, \ldots, m\}$ proportional to $\{\prod_{T: i\in T} \msg^{T\to i}_c\}$. Given a set of marginals, a natural way to extract an estimate for $\sigma$ is to round to the color in $\{1, \ldots, m\}$ at which the probability mass function is largest.
\end{definition}

Throughout we will make the following assumption that ensures that the trivial messages $\msg^{i\to S}_c = 1/m$ and $\msg^{S\to i}_c = 1/m$ are a fixed point, sometimes called the \emph{paramagnetic fixed point}, for the iteration above:

\begin{assumption}\label{assume:paramagnetic}
    The quantity $\sum_{\overline{\sigma}\in \{1,\ldots,m\}^{[k]}\backslash i} g(\overline{\sigma}\cup_i c)$ is constant across all $c\in \{1,\ldots,m\} $ and $i\in[k]$.
\end{assumption}

\begin{definition}\label{def:thresholds}
    Given $k,m,g$, the \emph{Kesten-Stigum} threshold $D_{\rm KS}$ is defined to be the largest average degree for which BP is locally stable around the paramagnetic fixed point, that is, starting from a small perturbation of the paramagnetic fixed point, it converges to the paramagnetic fixed point. More formally, $D_{\rm KS}$ is the largest average degree at which the Jacobian of the BP operator $\{\msg^{i\to S}[t]\}\mapsto \{\msg^{i\to S}[t+1]\}$ has spectral radius less than $1$.

    The \emph{condensation} threshold $D_{\rm cond}$ is defined to be the largest average degree at which the planted CSP ensemble and the following simple \emph{null model} become mutually contiguous and thus statistically indistinguishable as $N \to \infty$. The null model is defined as follows: there is no single unknown assignment, but instead for every ordered subset $S$ of $k$ variables, Nature independently samples an unknown local assignment $\sigma_S \in \{1,\ldots,m\}^S$, and the observation is included with probability $\phi / N^{k-1}$ if $g(\sigma_S) = 1$. 
\end{definition}

For $D_{\rm cond} < kP/N < D_{\rm KS}$, there exists some \emph{other} fixed point of the BP operator whose marginals, once rounded to an assignment, achieves strictly higher overlap than does BP with messages initialized randomly. The prediction is that in this regime, no efficient algorithm can achieve optimal recovery~\cite{krzakala2009hiding}.

\begin{conjecture}[1RSB cavity prediction]\label{conj:1rsb}
    Suppose $k, m, g$ satisfy Assumption~\ref{assume:paramagnetic}, and let $D_{\rm KS}$ and $D_{\rm cond}$ denote the associated Kesten-Stigum and condensation thresholds for the average degree. Then for all $P$ for which $D_{\rm cond} < kP/N < D_{\rm KS}$, the best overlap achieved by a computationally efficient algorithm for recovering $\sigma$ is strictly less than the best overlap achievable.
\end{conjecture}

\begin{proof}[Proof of Proposition~\ref{prop:csp}]
    At masking fraction $\alpha$ satisfying the bounds in the Proposition, with probability at least $\alpha^N \ge (1 - \gamma^{-1}D_{\rm KS}/N^{k-1})^N \ge \Omega(1)$ we have that all tokens corresponding to latents $x_{\pi(i)}$ get masked. Independently of this, the number of unmasked tokens among the observation tokens $\mathcal{O}_S$ is distributed as $\mathrm{Bin}(N(N-1)\cdots (N-k+1), 1 - \alpha)$, so by standard binomial tail bounds, with constant probability (depending on the gap between $D_{\rm cond}$ and $D_{\rm KS}$) this lies between $\gamma^{-1} D_{\rm cond}N/k$ and $\gamma^{-1} D_{\rm KS}N/k$. Furthermore, of these unmasked tokens in expectation $\gamma$ fraction of them correspond to observations for which the associated predicate evaluates to $1$. Conditioned on the above events, the masking problem thus reduces exactly to inference for a planted constraint satisfaction problem at average degree $D_{\rm cond} < D < D_{\rm KS}$, from which the Proposition follows.
\end{proof}

\section{Experimental details in Section~\ref{sec:hardness}}

\subsection{Experimental details in Section~\ref{sec:hardness_text}} \label{appendix:exp_detail_text}

\paragraph{$\pi$-learner configurations.} We consider two distributions of $\pi$ that interpolate between $\mathrm{Unif\,}(\mathbb{S}_L)$ where $\mathbb{S}_L$ denote the uniform distribution over all permutations of indices $\{0,1, \ldots, L-1\}$ and the point mass at the identical distribution: (Closer) and (Much-closer). To construct those distributions, we start from the identity permutation and perform a certain number of random swapping operations. Since $L\log(L)$ number of swaps results in a distribution that is very close to $\mathrm{Unif\,}(\mathbb{S}_L)$ \cite{bormashenko2011coupling}, we use $L/10$ and $\sqrt{L}$ swaps to construct the (Closer) and (Much-closer) distributions, respectively. For consistency, we repeat this sampling process three times.

\paragraph{Model and training configurations.} As explained in Section~\ref{sec:hardness_text}, to evaluate the scaling law of the $\pi$-learner, we can simply adapt the autoregressive training setup (a transformer with causal attention) by modifying the input to $\pi(x_0)$ and using a learnable positional embedding layer instead of RoPE. We borrow the training configurations from \cite{nie2024scaling}, which are also consistent with the TinyLlama \cite{zhang2024tinyllama} configurations. In particular, we use AdamW optimizer \cite{loshchilov2019decoupled}, setting $\beta_1 = 0.9$, $\beta_2 = 0.95$, and a weight decay of $0.1$ and $L=2048$.
A cosine learning rate schedule is applied, with a maximum learning rate of $4 \times 10^{-4}$ and a minimum learning rate of $4 \times 10^{-5}$. We also note that \textbf{unless otherwise specified, we maintain the same training configuration throughout the paper.}

\paragraph{Examining scaling laws.} We conduct IsoFLOP analysis \cite{hoffmann2022trainingcomputeoptimallargelanguage}. For a given number of FLOPs $C$, by varying the number of non-embedding parameters of transformers, we set the iteration numbers so that the total number of tokens observed by the model during training equals $C/6N$, following prior studies \cite{hoffmann2022trainingcomputeoptimallargelanguage, kaplan2020scaling}. We then select the smallest validation loss and set it as a data point.

\subsection{Experimental details in Section~\ref{sec:imbalance_error}}
\label{appendix:exp_detail_3_3}

\subsubsection{Experiment on L\&O-NAE-SAT distribution} We consider the L\&O-NAE-SAT distribution with $(N,P) = (20,280)$. For each example sequence from L\&O-NAE-SAT, we pad the last $212$ tokens with an additional token value of $2$. We employ a $19$M MDM with RoPE and a maximum sequence length of $512$. Then, this MDM is trained for $2\times 10^3$ iterations. To attain a proxy MDM for the Bayes optimal predictor, we further train it for $5 \times 10^4$ iterations. 

To measure the error across different tasks, we consider the following setup. For each $\ell \in [1, N-1]$, we randomly mask $\ell$ tokens in the latent positions and $\ell \times (P/N)$ tokens in the observed positions. Across all masked prediction positions, $\ell (1 + P/N)$, we measure the error for each position. For certainty, we repeat this process $1000$ times. The result in Figure~\ref{fig:scaling_laws} corresponds to the case when $\ell = 11$, and we observe the same tendency for other values of $\ell$.

\subsubsection{Experiment on text data}
We take a $170$M MDM pretrained with text data for a baseline model. To measure the performance imbalance between likelihood modeling tasks 

\begin{equation*}
    \mathbb{E}_{x_0 \sim p_{\rm{data}}}\left[\sum_{i=0}^{L-1} \log p_\theta \left( x_0^{\pi(i)} \Big| x_0 [\pi\{i,\ldots,L-1\}] \right) \right].
\end{equation*}
As done in the experiments in Section~\ref{sec:hardness_text}, we sample $\pi$s from 
three different distributions: $\mathrm{Unif}(\mathbb{S}_L)$, (Closer), the point mass of identical distribution. For each case, we calculate the expectation over $1024$ samples of $x_0 \sim p_{\rm{data}}$.

\section{Experimental details in Section~\ref{sec:inference}}

\subsection{Experimental details in Section~\ref{subsec:adaptive_inference}}
\label{appendix:exp_detail_inference}

\subsubsection{Experiment on L\&O-NAE-SAT distribution}
We consider five instances of L\&O-NAE-SAT: $(N,P) = (25,275), (30,270), (40,260), (50,250), (100,200)$. For each distribution, we train a 19M MDM and measure the accuracy difference between vanilla inference and adaptive inference using top probability margin.

\subsubsection{Experiment on text data}

\paragraph{Top probability margin sampler with temperature.}
To modify our inference for text data modeling, which does not have a determined answer, we found that adding a certain level of temperature to the oracle is useful. This is because the top probability margin or the top probability often leads to greedy sampling, which harms the diversity (entropy) of the generated samples. Therefore, we consider a variant of the oracle as follows, incorporating a Gaussian noise term $\epsilon$. 
\begin{align*}
    \mathcal{F}(\theta, x_t) = \text{Top } K \left(| p_\theta(x^i = j_1 | x_t) - p_\theta(x^i = j_2 | x_t) | + \epsilon \right).
\end{align*}
Note that this approach has also been employed for unconditional sampling \cite{wang2024diffusion,zheng2024reparameterized}.

\paragraph{Generative perplexity and entropy.} 
We employ a 1.1B MDM pretrained on text data as a baseline. For each sampling step, we unconditionally generate samples using both vanilla and adaptive inference. Next, we calculate the likelihood using LLama2-7B as a baseline large language model. Moreover, we denote the entropy of a generated sample $x$ as  $\sum p_i \log p_i$, where $p_i = \# \{x^i = i \}/L$.

\paragraph{Choice of number of tokens to unmask.} We set the number of tokens to unmask $K$ so that the number of unmasked tokens matches that of vanilla MDM inference in expectation. For an inference transition from step $t$ to $s$, vanilla MDM expects $(\#\text{ mask tokens in the current }x_t)\times\frac{\alpha_s-\alpha_t}{1-\alpha_t}$ unmasked. Accordingly, we choose $K = (\#\text{ mask tokens in the current }x_t)\times\frac{\alpha_s-\alpha_t}{1-\alpha_t}$. This choice keeps the number of revealed tokens balanced throughout inference. Alternatively, one can sample $K$ stochastically from $\mathrm{Binom}(\#\text{ mask tokens in the current }x_t,\frac{\alpha_s-\alpha_t}{1-\alpha_t})$. We found that both the deterministic and stochastic choices of $K$ result in comparable generative perplexity.

This choice of $K$ can be potentially helpful when the network is time-conditioned, since this keeps $(\# \text{mask tokens in the current $x_t$})\approx (1-\alpha_t)\times L$ where $L$ is the max sequence length--matching the marginal that the model saw during training.

\subsection{Experimental details on Sudoku and Zebra puzzles}
\label{appendix:sudoku-zebra-exp-details}

\paragraph{Dataset.} For both Sudoku and Zebra puzzles, we use the dataset provided in \citet{shah2024causal} to train our model. To evaluate our model on the same difficulty tasks, we use the test dataset proposed in \citet{shah2024causal}. This dataset is created by filtering the puzzles from \cite{david_g__radcliffe_2020} that can be solved using a fixed list of 7 strategies. To create a hard dataset to evaluate easy-to-hard generalization, we use the remaining puzzles from \cite{david_g__radcliffe_2020} as they either require a new strategy unseen during the training and/or require backtracking. The hard dataset contains around 1M Sudoku puzzles.

\paragraph{Model, training, and inference.} For the training and inference, we use the codebase of \cite{ye2024beyond} with keeping most of the hyperparameters default given in the codebase. For the Sudoku dataset, we use $6$M GPT-2 model, and for the Zebra dataset, we use $19$M model. We set the learning rate to 0.001 with a batch size of 128 to train the model for 300 epochs. For the inference, we use 50 reverse sampling steps using the appropriate strategy.  Additionally, we add Gumbel noise with a coefficient of 0.5 to the MDM inference oracle $\mathcal{F}$. 

\subsection{Experimental details on LLaDA-8B} \label{sec:llada_detail}

Our evaluation covers two task categories: (i) infilling(HumanEval-Infill and ROCStories) and (ii) instruction–answering (Math). For instruction–answering tasks, we employ a semi-autoregressive sampling strategy, whereas for infilling tasks we retain the non-autoregressive approach. For infilling tasks, the output length is predetermined—matching the size of the masked span—whereas instruction–answering tasks require an explicit length specification. For the latter, we follow the sampling configuration of \cite{nie2025large}.

For HumanEval-Infill, we adopt the problem set introduced by \cite{bavarian2022efficienttraininglanguagemodels}. Each instance is grouped by the span of the masked code—the region the model must infill—into three categories: \textit{single-line}, \textit{multi-line}, and \textit{split}. The task difficulty rises as the length of the masked span increases.

\section{Omitted proofs}
\label{appenix:mdm-equivalent-loss}

\begin{proof}[Proof of Proposition~\ref{prop:mdm_loss}]
We build on Proposition 3.1 from \cite{zheng2024maskeddiffusionmodelssecretly} to obtain the result of Proposition~\ref{prop:mdm_loss}. We first re-state the result from \cite{zheng2024maskeddiffusionmodelssecretly} for the case when the denoising network $p_\theta$ does not depend on the noise-scale $t$ explicitly. Let $x(n)$ be a sequence with $n$ tokens being masked from $x_0$, and $x^i(n)$ denotes the $i^{\textrm{th}}$ token value of the sequence $x(n)$. Let $\Tilde{q}(x(n) | x_0)$ be the probability distribution corresponding to randomly and uniformly masking $n$ tokens of $x_0$. 

\begin{proposition}[Proposition~3.1 of \cite{zheng2024maskeddiffusionmodelssecretly}]
For clean data $x_0$, let $\tilde{q}(x(n)\mid x_0)$ be the discrete forward process that randomly and uniformly masks $n$ tokens of $x_0$.
Suppose the noise schedules $\alpha_t$ satisfies $\alpha_0 = 0$ and $\alpha_1 = 1$. Then, the MDM training loss \eqref{eqn:mdm_loss} can be reformulated as
\begin{align} \label{appendix:prop_pf_1}
    \mathcal{L}_\theta = - \sum_{n=1}^L \;\; \mathop{\mathbb{E}}_{x(n)\sim \tilde{q}(\cdot\mid x_0)}\left[\frac{1}{n}\sum_{\ell:x^\ell(n) = 0} \log p_\theta(x_0^\ell\mid x(n))\right].
\end{align}
\end{proposition}
To obtain an alternative formulation of \eqref{appendix:prop_pf_1}, we expand the expectation $x(n)\sim \tilde{q}(\cdot\mid x_0)$. Since there are total $L$ positions of $x_0$, we have the probability assigned for each $x(n)$ equals $1/\binom{L}{n}$. Therefore, expanding the above equation with the expectation $x(n)$ and treating $x(n)$ as $x[M]$ for some set $M$ of size $n$, we obtain the result. 
\begin{align*}
      \mathcal{L}_\theta 
      = -\sum_{M \in [L], i \in M} \frac{1}{\binom{L}{|M|}} \cdot \frac{1}{|M|} \log p_\theta(x_0^\ell\mid x[M]).
\end{align*}
\end{proof}

\subsection{Equivalence between the MDM loss and any-order autoregressive loss}
\label{sec:mdm-aoarm}

In this section, we will demonstrate the equivalence for MDM loss and any-order autoregressive loss. In particular, for all $x_0$, we show
\begin{align*}
   -\mathop{\mathbb{E}}_{\pi\sim\mathrm{Unif}(\mathbb{S}_L)}\left[\sum_{j=0}^{L-1} \log p_\theta \left( x_0^{\pi(j)} \Big| x_0 [\pi\{j\},\ldots,\pi \{L-1\}] \right) \right] = -\sum_{ M\subseteq [L],i \in M}\frac{1}{\binom{L}{|M|}} \frac{1}{| M |} \displaystyle   \log p_\theta(x^i_0 | x_0[M]). 
\end{align*}
We now consider $\{\pi(j),\dots,\pi(L-1) \}=M \subseteq [L]$ and $\pi(j) =i $ and count the number of $\pi \in \mathbb S_L$ that induces a specific term $\log p_\theta(x_0^i | x_0[M])$. To induce the term, for a given $M\in [L]$ and $i \in M$, $\pi$ must satisfy
\begin{align*}
    \pi(j) = i, \quad  \{\pi(j),\dots,\pi(L-1) \}=M.
\end{align*}
The number of $\pi$ that satisfies above is $(L-|M|)! \times (|M|-1)!$. Using this and the number of total permutations is $L!$, we obtain the result.
\begin{align*}
&\mathop{\mathbb{E}}_{\pi\sim\mathrm{Unif}(\mathbb{S}_L)}\left[\sum_{j=0}^{L-1} \log p_\theta \left( x_0^{\pi(j)} \Big| x_0 [\pi\{j\},\ldots,\pi \{L-1\}] \right) \right] \\
=&\frac{1}{L!}\sum_{\pi \in \mathrm{Unif}(\mathbb{S}_L) } \sum_{j=0}^{L-1} \log p_\theta \left( x_0^{\pi(j)} \Big| x_0 [\pi\{j\},\ldots,\pi \{L-1\}] \right) \\
=& \frac{1}{L!}\sum_{M \in [L], i \in M} \big[\log p_\theta(x_0^i | x_0[M]) \times(L-1-|M|)! \times (|M|-1)! \big] \\
=& \sum_{M \in [L], i \in M} \frac{1}{\binom{L}{|M|}} \frac{1}{|M|} \log p_\theta(x_0^i | x_0[M]).
\end{align*}